\documentclass{article}

\renewcommand{\sfdefault}{phv}
\newcommand{\dif}{\mathrm{d}}
\newcommand{\E}{\mathop{\mathbb{E}}\displaylimits}
\newcommand{\softmax}{\mathop{\mathrm{softmax}}}
\newcommand{\conv}{\mathop{\mathrm{conv}}}
\newcommand{\emb}{\mathop{\mathrm{emb}}}
\newcommand{\dist}{\mathop{\mathrm{dist}}}
\newcommand{\mathspan}{\mathop{\mathrm{span}}}
\newcommand{\range}{\mathop{\mathrm{range}}}
\newcommand{\Cov}{\mathop{\mathrm{Cov}}}
\newcommand{\tr}{\mathop{\mathrm{tr}}}

\renewcommand{\P}{\mathbb{P}}

\usepackage{amsmath,amsfonts,amssymb,bm}
\usepackage{aliascnt}
\usepackage[colorlinks=true, linkcolor=fdblue, citecolor=fdblue, urlcolor=fdblue]{hyperref}
\usepackage{cleveref}

\newtheorem{theorem}{Theorem}[section]

\newaliascnt{proposition}{theorem} 
\newtheorem{proposition}[proposition]{Proposition}
\aliascntresetthe{proposition}

\newaliascnt{lemma}{theorem} 
\newtheorem{lemma}[lemma]{Lemma}
\aliascntresetthe{lemma}

\newaliascnt{corollary}{theorem} 
\newtheorem{corollary}[corollary]{Corollary}
\aliascntresetthe{corollary}

\newaliascnt{definition}{theorem} 
\newtheorem{definition}[definition]{Definition}
\aliascntresetthe{definition}

\newenvironment{proof}[1][Proof]{\par\noindent\emph{#1.} }{\hfill$\square$\par}

\crefname{theorem}{theorem}{theorems}
\crefname{proposition}{proposition}{propositions}
\crefname{lemma}{lemma}{lemmas}
\crefname{corollary}{corollary}{corollaries}
\crefname{definition}{definition}{definitions}

\Crefname{theorem}{Theorem}{Theorems}
\Crefname{proposition}{Proposition}{Propositions}
\Crefname{lemma}{Lemma}{Lemmas}
\Crefname{corollary}{Corollary}{Corollaries}
\Crefname{definition}{Definition}{Definitions}

\def\eqref#1{equation~\ref{#1}}

\def\1{\bm{1}}

\def\vzero{{\bm{0}}}

\def\vmu{{\bm{\mu}}}
\def\vtheta{{\bm{\theta}}}
\def\vphi{{\bm{\phi}}}

\def\vkappa{{\bm{\kappa}}}

\def\vb{{\bm{b}}}
\def\vc{{\bm{c}}}

\def\ve{{\bm{e}}}

\def\vs{{\bm{s}}}

\def\vu{{\bm{u}}}
\def\vv{{\bm{v}}}
\def\vw{{\bm{w}}}
\def\vx{{\bm{x}}}
\def\vy{{\bm{y}}}
\def\vz{{\bm{z}}}

\def\vW{{\bm{W}}}

\def\mA{{\bm{A}}}

\def\mI{{\bm{I}}}

\def\mO{{\bm{O}}}

\def\mR{{\bm{R}}}

\def\mU{{\bm{U}}}
\def\mV{{\bm{V}}}

\def\mY{{\bm{Y}}}

\DeclareMathAlphabet{\mathsfit}{\encodingdefault}{\sfdefault}{m}{sl}
\SetMathAlphabet{\mathsfit}{bold}{\encodingdefault}{\sfdefault}{bx}{n}

\usepackage{auxiliaries/PRIMEarxiv}

\title{How Do Flow Matching Models Memorize and Generalize in Sample Data Subspaces?}

\author{
\begin{minipage}[t]{0.48\textwidth}
\centering
Weiguo Gao \\
School of Mathematical Sciences,\\ 
School of Data Science\\
Fudan University\\
Shanghai, 200433, China\\
\texttt{wggao@fudan.edu.cn}
\end{minipage}%
\hfill
\begin{minipage}[t]{0.48\textwidth}
\centering
Ming Li \\
School of Mathematical Sciences,
Shanghai Key Laboratory for Contemporary Applied Mathematcis\\
Fudan University\\
Shanghai, 200433, China\\
\texttt{mingli23@m.fudan.edu.cn}
\end{minipage}
}

\begin{document}
\maketitle

\begin{abstract}
Real-world data is often assumed to lie within a low-dimensional structure embedded in high-dimensional space. In practical settings, we observe only a finite set of samples, forming what we refer to as the \emph{sample data subspace}. It serves an essential approximation supporting tasks such as dimensionality reduction and generation. A major challenge lies in whether generative models can reliably synthesize samples that stay within this subspace rather than drifting away from the underlying structure. In this work, we provide theoretical insights into this challenge by leveraging \emph{Flow Matching models}, which transform a simple prior into a complex target distribution via a learned velocity field. By treating the real data distribution as discrete, we derive analytical expressions for the optimal velocity field under a Gaussian prior, showing that generated samples \emph{memorize} real data points and represent the sample data subspace exactly. To generalize to suboptimal scenarios, we introduce the Orthogonal Subspace Decomposition Network (OSDNet), which systematically decomposes the velocity field into subspace and off-subspace components. Our analysis shows that the off-subspace component decays, while the subspace component \emph{generalizes} within the sample data subspace, ensuring generated samples preserve both proximity and diversity.
\end{abstract}

\section{Introduction}

The representation of real-world data in high-dimensional spaces is a fundamental topic in machine learning and data analysis. According to the widely accepted manifold hypothesis \citep{fefferman2016testing}, data such as images are assumed to reside on a low-dimensional structure embedded within a high-dimensional ambient space (e.g., pixel space). In practical scenarios, however, we only observe a finite sample set, forming what we term the \emph{sample data subspace}, an approximation to the full structure. This sample data subspace plays a critical role in data-driven tasks, as it provides a concrete basis for algorithms in dimensionality reduction \citep{mcinnes2018umap} and generation~\citep{yao2024manifold}, which rely on observed samples to approximate the underlying structure. Despite the foundational importance of this subspace, it remains uncertain whether generative models, which synthesize new data samples, can reliably generate samples that adhere to this data-constrained subspace rather than drifting away from the underlying structure.

In response to this challenge, we explore solutions within the framework of \emph{Flow Matching models} \citep{lipman2023flow,liu2023flow}. These models generate samples by progressively transforming a simple prior distribution into a complex target distribution, guided by an ordinary differential equation (ODE) governed by a learned velocity field. Departing from the common assumption of a continuous data distribution, we assume the real data distribution is discrete, composed only of the available samples. This reflects the finite nature of real data, where we approximate the low-dimensional structure with a limited set of points. 

By focusing on the dynamics of the generative process in Flow Matching models, we show that, under a standard Gaussian prior, the optimal velocity field leads to analytical expressions. These expressions can be interpreted as a softmax-weighted sum of the vectors pointing from the current point to the real data points, offering a new perspective on the geometry of the generation paths under the velocity field. To the best of our knowledge, this is hitherto underexplored. We show that under the optimal velocity field, the generated samples \emph{memorize} the real data points, thereby faithfully representing the sample data subspace.

To extend this analysis to broader scenarios, we propose the \emph{Orthogonal Subspace Decomposition Network (OSDNet)}, a network class that decomposes the suboptimal velocity field into two parts: one within the data subspace and one orthogonal to it, referred to as the off-subspace part. These two parts influence the generated samples' subspace and off-subspace components, respectively. Importantly, they can be decoupled in the training loss function to allow for independent optimization. We demonstrate that the off-subspace component decays over time, while the subspace component \emph{generalizes} within the sample data subspace. This ensures that generated samples remain close to the sample data subspace, preserving both proximity and diversity.

Our main contributions are summarized as follows:

\begin{itemize}
\item \textbf{Rigorous mathematical formulation of the optimal velocity field and the OSDNet network class of the suboptimal counterpart.} We derive an explicit analytical expression for the optimal velocity field. For the suboptimal counterpart, we propose the Orthogonal Subspace Decomposition Network (OSDNet), a network class that decomposes the velocity field into subspace and off-subspace components in a systematic manner.

\item \textbf{Memorization of real data points and geometric properties of generation paths under the optimal velocity field.} Leveraging the analytical expression for the optimal velocity field, we analyze the geometry of generation paths in two distinct scenarios: one with sparse, well-separated real data points, and another with a hierarchical structure. Through these geometric insights, we rigorously demonstrate that generated samples memorize the real data points.

\item \textbf{Approximation and generalization in the sample data subspace under the suboptimal velocity field, with an upper bound on the expected distance to real data points.} We show that the off-subspace components of generated samples decay during training, while the subspace components generalize beyond the real data points. This is further refined by a theoretical upper bound on the expected distance between generated samples and real data points, which decreases as the approximation of the suboptimal velocity field improves.
\end{itemize}

\section{Related Works}

In this section, we briefly discuss related works, focusing on flow-based generative models, memorization and generalization in generative models, and low-dimensional structure learning. For a more comprehensive literature review, including an overview of generative models and a comparative analysis with two closely related works, please refer to \cref{app:addition_literature_review}.

\paragraph{Flow-based generative models.} Flow-based generative models generate samples by applying a sequence of transformations that progressively map a simple source distribution, such as a standard Gaussian, to more complex data distributions. Among these, diffusion models \citep{song2019generative, ho2020denoising, song2021score, karras2022elucidating} employ stochastic differential equations (SDEs) to guide the transformation process. They rely on the score function, which is the gradient of the log-density of the data distribution, to capture the structure of the underlying data. Recently, Flow Matching models \citep{lipman2023flow, liu2023flow} have emerged as a promising alternative. These models use ODEs instead of SDEs to represent the generative process. By learning a velocity field that maps the source distribution directly to the target distribution, they offer both computational efficiency and greater theoretical clarity. Flow Matching models have proven successful in large-scale tasks requiring rapid and accurate inference, such as DALL\(\cdot\)E 3 \citep{betker2023improving}. Given these strengths, we focus on Flow Matching models in this paper.

\paragraph{Memorization and generalization in generative models.} Memorization occurs when generative models produce samples that closely replicate the training data, effectively recalling specific examples rather than capturing the broader underlying patterns. In contrast, generalization refers to a model's ability to generate new, previously unseen samples that extend beyond the training data while still conforming to the learned distribution. Recently, memorization in diffusion models has been studied both empirically \citep{somepalli2023diffusion, yoon2023diffusion} and theoretically \citep{li2024good}. Generalization has also been theoretically examined in diffusion models, particularly from the perspectives of implicit bias \citep{kadkhodaie2024generalization} and statistical bounds \citep{li2024generalization}. Despite these advances, the concepts of memorization and generalization remain relatively underexplored in the context of Flow Matching models. Furthermore, existing research tends to focus on either memorization or generalization separately \citep{kadkhodaie2024generalization, li2024generalization, li2024good}, without integrating them into a unified framework. Additionally, current theoretical work often assumes the existence of a continuous real distribution, an assumption that may not hold in practical scenarios.

\paragraph{Low-dimensional structure learning in generative models.} According to the manifold hypothesis, data such as images lies on a low-dimensional structure within a high-dimensional space. Recent studies identify three key ways in which diffusion models exploit low-dimensional structures. First, studies have shown that diffusion models effectively learn low-dimensional data distributions, such as invertible transformations of subspaces \citep{chen2023score} or mixtures of low-rank Gaussians \citep{wang2024diffusion}. This effectiveness is demonstrated through bounds on score estimation \citep{chen2023score} or connections between score learning and subspace clustering \citep{wang2024diffusion}. These assumptions share conceptual similarities with our work, and we will provide a detailed comparative analysis in \cref{app:addition_literature_review}. Second, diffusion models preserve manifold structures by enforcing orthogonal denoising while allowing mixing within the structure \citep{wenliang2022score}. Third, convergence bounds have been established for data distributions on compact sets, broadening the applicability of diffusion models to non-Lebesgue measures \citep{de2022convergence}. Despite progress in learning low-dimensional structures with diffusion models, similar research for Flow Matching models remains largely unexplored, motivating our work.

\paragraph{Notations.} Matrices are represented by bold capital letters, e.g., \(\mA\), vectors by bold lowercase letters, e.g., \(\vx\), and scalars by regular letters, e.g., \(t\). For a matrix \(\mA\), we denote its spectral norm by \(\|\mA\|_2\), its Frobenius norm by \(\|\mA\|_{\mathrm{F}}\), and its trace by \(\tr(\mA)\). We use \(\E\) for the expectation operator, \(\softmax\) for the softmax function, \(\conv\) for the convex hull, \(\emb\) for sinusoidal positional encoding (specific definitions are given later), \(\dist\) for distance between sets, \(\mathspan\) for the linear span of vectors, \(\range\) for the space spanned by the columns of a matrix, \(\Cov\) for the covariance matrix.

\section{Preliminaries}

In this section, we provide background on Flow Matching models, including their foundational principles, key formulations, and the role of optimal transport paths in guiding the generative process.

\subsection{Flow Matching Models}

Flow Matching (FM) \citep{lipman2023flow} is a framework that transforms a simple prior distribution \( p_0(\vx) \), typically a standard Gaussian, into a complex target distribution \( p_1(\vx) \) by learning a velocity field. FM employs various probability paths, each defined by a time-dependent velocity field \( \vu_t(\vx) \). This field controls the evolution of the \emph{flow} \( \vphi_t(\vx) \) via the following ODE
\[
\frac{\dif \vphi_t(\vx)}{\dif t} = \vu_t(\vphi_t(\vx)), \quad \vphi_0(\vx) = \vx, \quad t\in[0, 1].
\]
The evolution of the probability density \( p_t(\vx) \), starting from \(p_0(\vx)\), is governed by the continuity equation (also known as the transport equation):
\[
\frac{\partial p_t(\vx)}{\partial t} + \nabla \cdot \big( p_t(\vx) \vu_t(\vx) \big) = \vzero, \quad t\in[0, 1].
\]
To approximate \( \vu_t(\vx) \), FM uses a neural network \( \vv_t(\vx; \vtheta) \), where \(\vtheta\) represents the network parameters. The network is trained by minimizing the following FM loss
\[
\mathcal{L}_{\text{FM}}(\vtheta) = \E_{t \sim \mathcal{U}[0,1], \vx \sim p_t(\vx)} \big[ \|\vv_t(\vx; \vtheta) - \vu_t(\vx)\|_2^2 \big].
\]
When the exact form of \( p_1(\vx) \) is unknown and only samples from \( p_1(\vx) \) are available, Conditional Flow Matching (CFM) constructs conditional paths \( p_t(\vx|\vx_1) \) for each data sample \( \vx_1 \sim p_1(\vx_1) \). The CFM loss is then defined as
\[
\mathcal{L}_{\text{CFM}}(\vtheta) = \E_{t \sim \mathcal{U}[0,1], \vx_1 \sim p_1(\vx_1), \vx \sim p_t(\vx|\vx_1)} \big[ \|\vv_t(\vx; \vtheta) - \vu_t(\vx|\vx_1)\|_2^2 \big].
\]
Assume that the conditional distribution is Gaussian, given by \( p_t(\vx|\vx_1) = \mathcal{N}(\vx;\vmu_t(\vx_1), \sigma_t(\vx_1)^2\mI) \), and that the flow takes the form
\[
\vphi_t(\vx) = \sigma_t(\vx_1)\vx + \vmu_t(\vx_1).
\]
The conditional velocity field becomes
\[
\vu_t(\vx|\vx_1) = \dfrac{\sigma_t'(\vx_1)}{\sigma_t(\vx_1)}\cdot\big(\vx-\vmu_t(\vx_1)\big) + \vmu_t'(\vx_1).
\]
FM models use two primary types of probability paths: diffusion paths, which are essentially equivalent to probability flow ODEs \citep{song2021score}, and optimal transport (OT) paths, which aim to transform distributions in the most efficient manner. In this paper, we focus on OT paths due to their wide application, theoretical elegance and appealing mathematical properties. We maintain that certain analytical results derived for OT paths can be extended to more general probability paths, providing a broader applicability of our findings. For more details regarding FM models, please refer to \cref{app:flow_matching_models}.

\subsection{Optimal Transport Paths}

OT paths are a specific instance of probability paths, where the flow is defined by linear interpolation between the initial and final positions. The mean function is given by \(\vmu_t(\vx_1) = t \vx_1\), and the standard deviation function is \(\sigma_t(\vx) = 1 - t\). The corresponding velocity field is 
\[
\vu_t(\vx|\vx_1) = \frac{\vx_1 - \vx}{1-t},
\]
resulting in the flow
\[
\vphi_t(\vx) = (1-t)\vx+t\vx_1
\]
when conditioned on \(\vx_1\). In this setup, the conditional flow follows straight-line paths with constant speed, which minimizes the complexity of the trajectory compared to diffusion paths. This leads to a simplified regression task for the CFM model, where the loss function becomes
\[
\mathcal{L}_{\text{CFM, OT}}(\vtheta) = \E_{t\sim\mathcal{U}[0,1], \vx_1\sim p_1(\vx_1), \vx_0\sim p_0(\vx_0)}\big[\|\vv_t(\vphi_t(\vx_0);\vtheta)-(\vx_1-\vx_0)\|_2^2\big].
\]
Note that this type of flow is also referred to as Rectified Flow in \citep{liu2023flow}.

\section{The Optimal Velocity Field}

In this section, we provide a detailed analysis of the derivation of the optimal velocity field and explore its role in shaping the geometry of generation paths. Building on these geometric insights, we prove that the generated samples memorize real data points and represent the data-defined subspace exactly.

\subsection{Derivation of the Optimal Velocity Field}

We begin by presenting the optimal velocity field that minimizes the CFM loss. The proof of \cref{prop:optimal_velocity_field}, along with all the other proofs in this paper, is provided in \cref{app:proofs_to_theorems}.

\begin{proposition}[Optimal velocity field for conditional flow matching]
\label{prop:optimal_velocity_field}
Assume that the function class \(\{\vv_t(\vx;\vtheta)\}\) has enough capacity. Then the optimal velocity field \( \vv_t^*(\vx) \) which minimizes the Conditional Flow Matching (CFM) loss
\[
\mathcal{L}_{\text{CFM}}(\vtheta) = \E_{t \sim \mathcal{U}[0,1], \vx_1 \sim p_1(\vx_1), \vx \sim p_t(\vx|\vx_1)} \big[ \|\vv_t(\vx; \vtheta) - \vu_t(\vx|\vx_1)\|_2^2 \big],
\]
is given by
\[
\vv_t^*(\vx) = \E_{\vx_1 \sim p_1(\vx_1)}[ \vu_t(\vx|\vx_1)|\vx, t].
\]
\end{proposition}

Having established the general form of the optimal vector field, we now investigate a specific case where the prior \(p_0\) is standard Gaussian and the real distribution \(p_1\) is a discrete distribution.

\begin{theorem}[Optimal velocity field for discrete target distribution]
\label{thm:optimal_velocity_field_expression}
Let \( p_0 \sim \mathcal{N}(\vzero, \mI_d) \), and suppose \( p_1 \) is a discrete distribution over a set of points \( \{\vy^i\colon 1\leq i\leq N\}\subset \mathbb{R}^d \), given by
\[
p_1 \sim \frac{1}{N} \sum_{i=1}^{N} \delta_{\vy^i},
\]
where \( \delta_{\vy^i} \) denotes the Dirac delta function centered at \( \vy^i \). Assume that the conditional velocity field is
\[
\vu_t(\vx|\vx_1) = \dfrac{\sigma_t'(\vx_1)}{\sigma_t(\vx_1)}\cdot\big(\vx-\vmu_t(\vx_1)\big) + \vmu_t'(\vx_1).
\]
Then the optimal velocity field \( \vv_t^*(\vx) \) that minimizes the CFM loss is given by
\[
\vv_t^*(\vx) = \sum_{i=1}^{N} \Big(\dfrac{\sigma_t'(\vy^i)}{\sigma_t(\vy^i)}\cdot\big(\vx-\vmu_t(\vy^i)\big) + \vmu_t'(\vy^i)\Big)\cdot\dfrac{\exp\big(-\|(\vx-\vmu_t(\vy^i))/\sigma_t(\vy^i)\|_2^2/2\big)}{\sum_{j=1}^{N}\exp\big(-\|(\vx-\vmu_t(\vy^j))/\sigma_t(\vy^j)\|_2^2/2\big)}.
\]
Specifically, in the case of OT paths where \(\vmu_t(\vx_1)=t\vx_1\) and \(\sigma_t(\vx)=1-t\), the optimal velocity field simplifies to
\[
\vv_t^*(\vx) = \sum_{i=1}^{N}\dfrac{\vy^i-\vx}{1-t}\cdot\dfrac{\exp\big(-\|(\vx-t\vy^i)/(1-t)\|_2^2/2\big)}{\sum_{j=1}^{N}\exp\big(-\|(\vx-t\vy^j)/(1-t)\|_2^2/2\big)}. 
\]
\end{theorem}

We now discuss the implications of \cref{thm:optimal_velocity_field_expression}. The optimal velocity field \( \vv_t^*(\vx) \) is a weighted sum of the directions \( (\vy^i-\vx)/(1-t) \), where these directions lie within the subspace spanned by \( \{\vy^1, \dotsc, \vy^N, \vx\} \). Although \( \vx \) changes over time, the entire generation path remains confined to the subspace \( \mathspan\{\vy^1, \dotsc, \vy^N, \vx_0\} \), with \( \vx_0 \) denoting the initial noise sample. This observation implies that the generation path is constrained to a fixed low-dimensional subspace defined by the initial conditions and the real data points.

The weighting in this expression takes the form of a \emph{softmax} function, where the softmax weights are determined by the Euclidean distance between the generated point \( \vx \) and \( t\vy^i \), scaled by the time-dependent factor \(1/(1-t)\). As a result, at time \(t\), real data points closer to \( \vx/t \)\footnote{When \(t=0\), all the real data points contribute equally.} exert a stronger influence on the velocity field, while the contribution from more distant points decays exponentially. As the generative process evolves, this weighting becomes increasingly concentrated around the nearest real data points, leading the generated samples to progressively approach the real data points. 

We remark that the optimal velocity field \( \vv_t^*(\vx) \) coincides with \(\vu_t(\vx)\) as defined in \citep{lipman2023flow}. To the best of our knowledge, this equivalence has not been explicitly addressed in previous works. As an immediate consequence, we have the following corollary based on \citep[Theorem 1]{lipman2023flow}, which provides the marginal probability path under the optimal velocity field. This result will be useful in the subsequent subsections.

\begin{corollary}[Marginal probability path induced by the optimal velocity field \citep{lipman2023flow}]
\label{cor:marginal_probability_path}
Let \( p_0 \sim \mathcal{N}(\vzero, \mI_d) \), and suppose \( p_1 \) is a discrete distribution over a set of points \( \{\vy^i\colon 1\leq i\leq N\}\subset \mathbb{R}^d \), given by
\[
p_1 \sim \frac{1}{N} \sum_{i=1}^{N} \delta_{\vy^i}.
\]
Then the optimal velocity field \( \vv_t^*(\vx) \) for OT paths, as derived in \cref{thm:optimal_velocity_field_expression}, induces the marginal probability path 
\[
p_t(\vx) = \int_{\mathbb{R}^d} p_t(\vx | \vx_1) p_1(\vx_1) \dif \vx_1 = \dfrac{1}{N}\sum_{i=1}^{N} \mathcal{N}(\vx; t\vy^i, (1 - t)^2\mI_d).
\]
\end{corollary}

\subsection{Path Geometry under the Optimal Velocity Field}
\label{subsec:path_geometry_under_the_optimal_velocity_field}

In this subsection, we investigate the geometry of generation paths under the optimal velocity field. We aim to understand how various properties of the real data points shape these paths and how the velocity field guides the generated points towards the real data points. By examining the intermediate behavior of \( \vx_t \), we seek to gain intuitive insights into the memorization phenomenon. We focus on two scenarios: (\romannumeral 1) when the real data points are sparse and well-separated, and (\romannumeral 2) when they exhibit a hierarchical structure. These cases result in distinct generation path behaviors, as formalized in \cref{thm:path_geometry_straight,thm:path_geometry_hierarchical}.

We first consider a simple setup where the real data points are sparse and well-separated. In this case, we expect generated samples to move directly toward the nearest real data point as the generation progresses. The key intuition is that the softmax weights controlling the velocity field quickly concentrate on a single data point, resulting in nearly linear generation paths directed toward that point. \Cref{thm:path_geometry_straight} formalizes this observation and provides a probabilistic bound on how quickly the softmax weights concentrate on individual points. Please refer to \cref{fig:path_geometry_straight} for the conceptual illustration. We also provide the visualization of generation paths under such circumstances in \cref{app:numerical_experiments}.

\begin{figure}[htb]
\centering
\includegraphics[width=\textwidth]{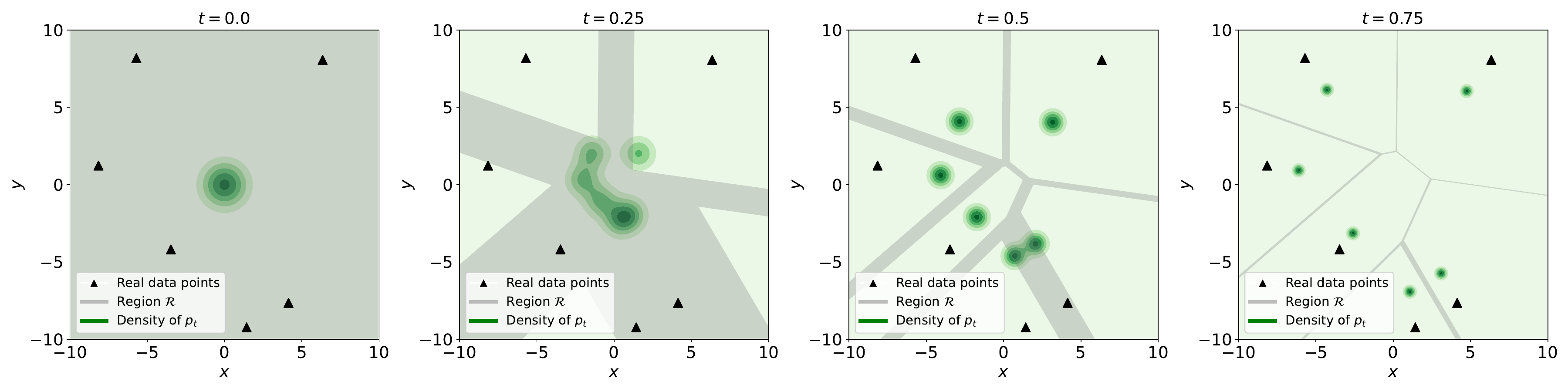}
\caption{The conceptual illustration of \cref{thm:path_geometry_straight}, which shows the regions where the softmax weights in the optimal velocity field are not close to a one-hot vector. The black triangles represent the sparse and well-separated real data points. The gray areas indicate regions where the largest entry of the softmax weight is less than \(0.99\). The darkness of the green areas reflects the densities of \(p_t\). Initially, the probability that \(x \sim p_t\) falls into the gray region is high (at \(t=0\)), but it decreases rapidly over time. This demonstrates that the softmax weights quickly focus on a single data point as the generation process evolves. The core of \cref{thm:path_geometry_straight} lies in estimating this probability.}
\label{fig:path_geometry_straight}
\end{figure}

\begin{theorem}[Probability bound for softmax weight concentration]
\label{thm:path_geometry_straight}
Let \( \{\vy^i\colon 1\leq i\leq N\}\subset \mathbb{R}^d \) be real data points such that \( \|\vy^i - \vy^j\|_2 \geq M \) for all \( 1\leq i\neq j\leq N \). Consider the softmax weight function:
\[
\vw_t(\vx) = \softmax\Big(-\dfrac{1}{2}\Big\Vert\dfrac{\vx - t\vy^1}{1 - t}\Big\Vert_2^2, -\dfrac{1}{2}\Big\Vert\dfrac{\vx - t\vy^2}{1 - t}\Big\Vert_2^2, \dotsc, -\dfrac{1}{2}\Big\Vert\dfrac{\vx - t\vy^N}{1 - t}\Big\Vert_2^2\Big).
\]
Assume that
\[
\vx \sim p_t = \dfrac{1}{N}\sum_{i=1}^{N} \mathcal{N}(t\vy^i, (1 - t)^2\mI_d).
\]
Then, the probability that \(\vx\) falls into the region \(\mathcal{R} \coloneqq \{\vx\colon \max(\vw_t(\vx)) \leq \tau < 1\}\), i.e., where the softmax weight vector is not close to a one-hot vector in the \(\tau\)-sense, is bounded by
\[
\P(\vx\in\mathcal{R})\leq
\dfrac{1}{\sqrt{2\pi}}\cdot\dfrac{1-t}{t}\cdot\dfrac{1}{M}\cdot\log\dfrac{\tau(N-1)}{1-\tau}\cdot N(N-1).
\]
\end{theorem}

\Cref{thm:path_geometry_straight} highlights several important aspects regarding how the softmax weights concentrate on one component: 
\begin{itemize}
\item As \(t\) approaches 1, the probability of the softmax weight vector not being concentrated drops rapidly to zero, implying that the weights become more focused on a single component as \(t\) increases. 
\item The parameter \(\tau\) controls the level of concentration: larger values of \(\tau\) indicate stronger concentration, with the probability of non-concentration decreasing logarithmically as \(\tau\) increases.
\item The separation \(M\) between real data points inversely affects this probability; larger values of \(M\) imply greater separation between components, which leads to a lower probability that the softmax vector is not close to a one-hot vector.
\item As the number of components \(N\) increases, the probability of the softmax weight vector not being one-hot grows quadratically, indicating that a higher number of components requires greater separation to maintain concentration.
\item Crucially, the bound is independent of the dimensionality \(d\), implying that this concentration behavior remains robust even in high-dimensional spaces.
\end{itemize}

Together, \cref{thm:path_geometry_straight} shows that when the real data points are sparse and well-separated, the generation paths tend to move directly toward individual real data points, leading to memorization as the softmax weights increasingly favor one point near the end of the generation process.

Next, we explore more complex data distributions where the real data points are hierarchically structured. In such cases, generation paths first distinguish broad clusters of points before refining to finer ones, which is formalized in \cref{thm:path_geometry_hierarchical}. Please refer to \cref{fig:path_geometry_hierarchical} for the conceptual illustration. We also provide the visualization of generation paths in \cref{app:numerical_experiments}.

\begin{figure}[htb]
\centering
\includegraphics[width=\textwidth]{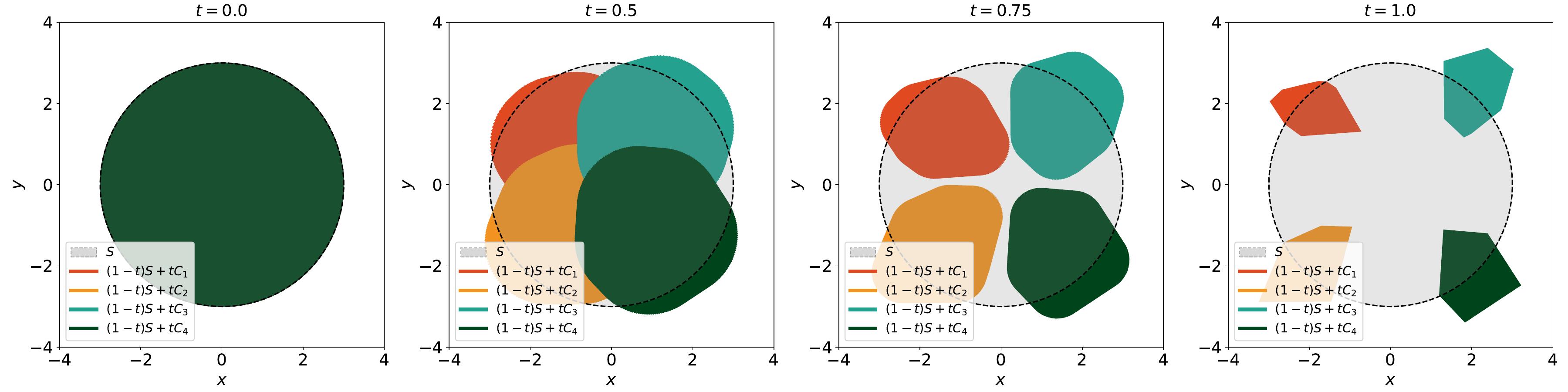}
\caption{The conceptual illustration of \cref{thm:path_geometry_hierarchical}, which depicts the hierarchy emergence. The gray disk represents the set \( S \) where \( p_0 \) is supported, while the colored regions at \( t = 1 \) indicate the support of \( p_1 \). At intermediate stages, the irregularly shaped regions correspond to \( (1-t)S + tC_i \). Initially, at \( t = 0 \), all regions coincide with \( S \), forming a mixed area. As time progresses, the regions start to separate, and once a generation path enters a specific region, it remains confined to it. The main result of \cref{thm:path_geometry_hierarchical} is the proof of a separation time, strictly before \( t = 1 \), when the regions \( (1-t)S + tC_i \) become fully distinct.}
\label{fig:path_geometry_hierarchical}
\end{figure}

\begin{theorem}[Hierarchy emergence of generation paths]
\label{thm:path_geometry_hierarchical}
Let \(C_{i, j}\;(i \in \mathcal{I},\, j \in \mathcal{J}_i)\) be a collection of disjoint, bounded, closed convex sets where the distribution \(p_1\) is supported. Define
\[
C_i = \conv\Big(\bigcup_{j \in \mathcal{J}_i} C_{i, j}\Big), \quad i \in \mathcal{I},
\]
where \(\conv(\cdot)\) denotes the convex hull of a set. Assume that the sets \(C_i\) are disjoint, and that the distribution \(p_0\) is supported on a bounded, closed convex set \(S\). 

Then, there exists \(t_1 \in (0, 1)\) such that for any \(t \in (t_1, 1]\), the \(|\mathcal{I}|\) convex sets
\[
(1 - t) S + t C_i, \quad i \in \mathcal{I}
\]
are disjoint. Similarly, there exists \(t_2 \in (t_1, 1)\) such that for any \(t \in (t_2, 1]\), the \(\sum_{i \in \mathcal{I}} |\mathcal{J}_i|\) convex sets
\[
(1 - t) S + t C_{i, j}, \quad i \in \mathcal{I},\, j \in \mathcal{J}_i
\]
are disjoint.

Let \(\vphi_t\) denote the optimal flow under the OT paths. Then, for any \(\vx\) drawn from \(p_0\), if \(\vphi_{t_1}(\vx) \in (1 - t_1) S + t_1 C_i\), it follows that \(\vphi_t(\vx) \in (1 - t) S + t C_i\) for any \(t \in (t_1, 1]\). Specifically, \(\vphi_1(\vx) \in C_i\). Furthermore, if \(\vphi_{t_2}(\vx) \in (1 - t_2) S + t_2 C_{i,j}\), then \(\vphi_t(\vx) \in (1 - t) S + t C_{i,j}\) for any \(t \in (t_2, 1]\). Specifically, \(\vphi_1(\vx) \in C_{i,j}\).
\end{theorem}

We provide several clarifications regarding the assumptions and results of \cref{thm:path_geometry_hierarchical}.

\begin{itemize}
\item The bounded, closed convex sets \(C_{i,j}\)'s can be isolated points, while the sets \(C_i\) can be the convex hulls of these points, forming polygons. See \cref{fig:path_geometry_hierarchical} for an illustration under this configuration.

\item The assumption that \(p_0\) is supported on a bounded convex set \(S\) can be approximately satisfied when \(p_0\) is a standard Gaussian distribution. In high dimensions, most of the probability mass in Gaussian distributions is concentrated near a thin shell, akin to a ``soap bubble''\footnote{See \url{https://www.inference.vc/high-dimensional-gaussian-distributions-are-soap-bubble/} for details.}, effectively approximating support on a bounded, closed convex set.

\item The theorem aligns with empirical findings in diffusion models, where coarse features are generated first, followed by finer details \citep{wang2023painters}. This matches the separation of broad clusters \(C_i\) first, then finer clusters \(C_{i,j}\), reflecting the progressive refinement of image details.

\end{itemize}

\Cref{thm:path_geometry_hierarchical} shows that when the real data points exhibit a hierarchical structure, the generation paths first distinguish broad clusters before refining to smaller ones. This progression leads to memorization, as the paths increasingly focus on individual points within the clusters, ultimately capturing specific real data points at the end of the generation process.

\subsection{Memorization Phenomenon under the Optimal Velocity Field}

We are now prepared to formally state the memorization phenomenon under the optimal velocity field. One might question why we need to restate and prove this result, as \cref{cor:marginal_probability_path} already demonstrates that \( p_1(\vx) \) corresponds to the discrete distribution of real data points. We point out that while \cref{cor:marginal_probability_path} relies on the marginal distribution and tools such as the continuity equation to establish the result, it does not provide insight into the behavior of the individual generation paths. In contrast, \cref{thm:memorization} directly focuses on the flow ODE, offering a more direct and transparent analysis. Moreover, the technique of isolating the linear term and applying the method of variation of constants will prove especially useful in the next section.

\begin{theorem}[Memorization phenomenon under the optimal velocity field]
\label{thm:memorization}
Consider the flow ODE
\[
\dfrac{\dif\vphi_t(\vx)}{\dif t} = \vv_t^*(\vphi_t(\vx)), \quad \vphi_0(\vx)=\vx, \quad t\in[0, 1],
\]
where
\[
\vv_t^*(\vx) = \sum_{i=1}^{N}\dfrac{\vy^i-\vx}{1-t}\cdot\dfrac{\exp\big(-\|(\vx-t\vy^i)/(1-t)\|_2^2/2\big)}{\sum_{j=1}^{N}\exp\big(-\|(\vx-t\vy^j)/(1-t)\|_2^2/2\big)}.
\]
Then the set \(\{\vphi_1(\vx)\colon \vx\in\mathbb{R}^d\}\) equals \(\{\vy^1, \vy^2, \dotsc, \vy^N\}\) up to a finite set.
\end{theorem}

In \cref{thm:memorization}, the expression ``up to a finite set'' is necessary and cannot be omitted. For example, in a simple scenario with only two real data points, if a noise sample starts at the midpoint between these two points, it will remain stationary under \(\vv_t^*\), similar to a stationary force field. However, we emphasize that such occurrences have zero probability since the noise samples are drawn from a standard Gaussian distribution. Therefore, despite this technicality, we still refer to this phenomenon as memorization.

\section{The Suboptimal Velocity Field}

In the previous section, we explore the theoretical properties of the optimal velocity field. However, in practical applications, neural networks are commonly used to approximate the velocity field, resulting in a suboptimal one. To better capture and analyze the resulting effects, we develop a network class called the Orthogonal Subspace Decomposition Network (OSDNet) to model the suboptimal velocity field. With OSDNet, we show that under the suboptimal velocity field, off-subspace components decay, while the components within the subspace demonstrate generalization behavior, ensuring generated samples remain close to the data-defined subspace while maintaining diversity.

\subsection{Orthogonal Subspace Decomposition Network}

We develop a network class called Orthogonal Subspace Decomposition Network (OSDNet) to model the suboptimal velocity field. First, we present the formal definition, followed by a detailed explanation of the rationale behind it. Please refer to \cref{fig:OSDNet} for an illustration.

\begin{definition}[Orthogonal Subspace Decomposition Network (OSDNet)]
Let \( \{\vy^i\colon 1\leq i\leq N\}\subset \mathbb{R}^d \) be real data points, which form the columns of the matrix \(\mY = [ \vy^1, \vy^2, \dotsc, \vy^N ]\). Let \( \mV\in \mathbb{R}^{d\times D}\) be a matrix with orthonormal columns spanning the same subspace as these data points, obtained by performing a reduced singular value decomposition on \(\mY\), i.e., \(\mY = \mV\mR\), where \(\mR\in\mathbb{R}^{D\times N}\) is the product of a diagonal matrix (containing the non-zero singular values of \(\mY\)) and a matrix with orthonormal rows (containing the right singular vectors). The suboptimal velocity field, as defined by OSDNet, is given by
\[
\hat{\vv}_t\big(\vx; \hat{\mO}_t, \hat{\vs}_t(\vx)\big) = \mV^\perp \cdot \hat{\mO}_t \cdot (\mV^\perp)^\top \cdot \vx + \mV \cdot \hat{\vs}_t(\vx),
\]
where \( \mV^\perp \) denotes the orthogonal complement of \( \mV \), \( \hat{\mO}_t \) is a time-dependent diagonal matrix, and \( \hat{\vs}_t(\vx) \) is a time-dependent vector function representing the component of the velocity field within the subspace spanned by \(\mV\).
\end{definition}

\begin{figure}[htb]
\centering
\includegraphics[width=.7\textwidth]{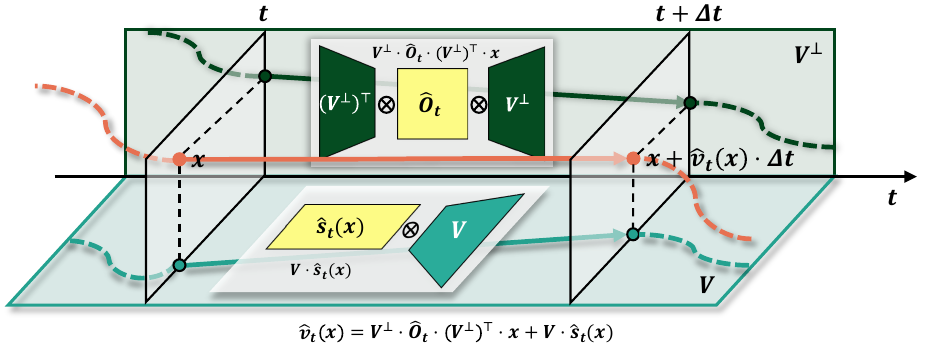}
\caption{An illustration of the Orthogonal Subspace Decomposition Network (OSDNet) in a 2-dimensional space, where time \(t\) progresses from left to right with an interval of \(\Delta t\). The velocity vector \(\hat{\vv}_t(\vx)\) (shown by the \textcolor[rgb]{0.88, 0.29, 0.13}{red} arrow pointing from \(\vx\) to \(\vx + \hat{\vv}_t(\vx) \cdot \Delta t\)) is decomposed into two components: one within the subspace \(\mV\) (\textcolor[rgb]{0.14, 0.63, 0.56}{cyan}) and the other in the orthogonal complement \(\mV^\perp\) (\textcolor[rgb]{0.0, 0.27, 0.11}{green}). These components, \(\mV \cdot \hat{\vs}_t(\vx)\) and \(\mV^\perp \cdot \hat{\mO}_t \cdot (\mV^\perp)^\top \cdot \vx\), are computed by OSDNet, represented within the white background. The trapezoidal shapes in the network indicate the input and output dimensions, while the symbol \(\otimes\) represents matrix multiplication. The dotted curves illustrate the trajectory \(\vphi_t(\vx)\) before time \(t\) and after time \(t + \Delta t\).}
\label{fig:OSDNet}
\end{figure}

This definition is grounded in the principle that the ranges of \(\mV\) and \(\mV^\perp\) span the entire space. As a result, we can project the suboptimal velocity field \(\hat{\vv}_t(\vx)\) onto the subspaces defined by \(\mV\) and \(\mV^\perp\), respectively. The component within \(\mV\) naturally takes the form \(\mV \cdot \hat{\vs}_t(\vx)\), whereas the term within \(\mV^\perp\), specifically \(\mV^\perp \cdot \hat{\mO}_t \cdot (\mV^\perp)^\top \cdot \vx\), might initially appear unconventional. However, this formulation is motivated by the structure of the optimal velocity field, which we explore in \cref{prop:optimal_velocity_field_as_a_specific_instance_of_OSDNet}.

\begin{proposition}[Optimal velocity field as a specific instance of OSDNet]
\label{prop:optimal_velocity_field_as_a_specific_instance_of_OSDNet}
The optimal velocity field \(\vv_t^*(\vx)\) derived in \cref{thm:optimal_velocity_field_expression} is a specific instance of the OSDNet, where
\[
\hat{\mO}_t^* = -\dfrac{1}{1 - t} \cdot \mI_{d-D}
\]
and
\[
\hat{\vs}_t^*(\vx) = \dfrac{1}{1 - t} \cdot \big(\mR\cdot\vw_t(\vx) - \mV^\top \cdot \vx \big).
\]
\end{proposition}

We now elaborate on the assumption that \(\hat{\mO}_t\) is diagonal. Using the method of variation of constants, we can expression the solution to the ODE 
\[
\dfrac{\dif \vphi_t(\vx)}{\dif t} = \hat{\vv}_t\big(\vphi_t(\vx); \hat{\mO}_t, \hat{\vs}_t(\vx)\big) = \mV^\perp \cdot \hat{\mO}_t \cdot (\mV^\perp)^\top \cdot \vphi_t(\vx) + \mV \cdot \hat{\vs}_t(\vphi_t(\vx)), \quad \vphi_0(\vx)=\vx, \quad t\in [0, 1],
\]
as
\[
\vphi_t(\vx) = \mathcal{T}\exp\Big(\int_{0}^{t} \mV^\perp\cdot\hat{\mO}_s\cdot(\mV^\perp)^\top \dif s \Big) \cdot \bigg(\vx + \int_{0}^{t} \mathcal{T}\exp\Big( \int_{0}^{s} \mV^\perp\cdot \hat{\mO}_\tau\cdot(\mV^\perp)^\top \dif \tau \Big)^{-1} \cdot \mV \cdot \hat{\vs}_s(\vphi_s(\vx)) \dif s \bigg),
\]
where \(\mathcal{T}\exp\) denotes the time-ordered matrix exponential\footnote{Please refer to \cref{app:time_ordered_exponential} for background on time-ordered exponential.}. To ensure that the matrices \(\{\mV^\perp\cdot\hat{\mO}_t\cdot(\mV^\perp)^\top\colon 0\leq t\leq 1\}\) commute with each other, and thus making the time-ordered matrix exponentials tractable, we assume that \(\{\hat{\mO}_t\colon 0\leq t\leq 1\}\) are diagonal. Under these conditions, we can decompose the solution \(\vphi_1(\vx)\) (i.e., the generated sample) into two components: one lying within the span of \(\mV\), and the other within the span of \(\mV^\perp\).

We begin by computing
\[
\exp\Big( \int_{0}^{1} \mV^\perp\cdot\hat{\mO}_s\cdot (\mV^\perp)^\top \dif s \Big).
\]
Note that
\[
\mV^\perp\cdot\hat{\mO}_s\cdot (\mV^\perp)^\top =
\begin{bmatrix}
\mV & \mV^\perp
\end{bmatrix}
\cdot
\begin{bmatrix}
\vzero & \vzero \\
\vzero & \hat{\mO}_s\\
\end{bmatrix}
\cdot
\begin{bmatrix}
\mV^\top \\
(\mV^\perp)^\top
\end{bmatrix}.
\]
By the definition of the matrix exponential, we obtain
\[
\exp\Big( \int_{0}^{1} \mV^\perp\cdot\hat{\mO}_s\cdot (\mV^\perp)^\top \dif s \Big) = \mV\mV^\top + \mV^\perp\cdot\exp\Big(\int_0^1\hat{\mO}_s\dif s\Big)\cdot 
(\mV^\perp)^\top.
\]
Next, we compute
\[
\exp\Big( \int_{0}^{1} -\mV^\perp \cdot \hat{\mO}_s \cdot (\mV^\perp)^\top \dif s \Big)
\]
in a similar way and get
\[
\exp\Big( \int_{0}^{1} \mV^\perp \cdot \hat{\mO}_s \cdot (\mV^\perp)^\top \dif s \Big)^{-1} = \mV\mV^\top + \mV^\perp \cdot \exp\Big(\int_{0}^{1} -\hat{\mO}_s\dif s\Big)\cdot  (\mV^\perp)^\top.
\]
Substituting these expressions into \(\vphi_1(\vx)\), and noting that \(\mV\) and \(\mV^\perp\) have orthonormal columns, we arrive at the decomposition
\[
\vphi_1(\vx) = \underbrace{\mV\cdot\Big(\mV^\top\cdot \vx+\int_{0}^{1}\hat{\vs}_s(\vphi_s(\vx))\dif s\Big)}_{\text{within the span of }\mV} + \underbrace{\mV^\perp\cdot \exp\Big(\int_0^1\hat{\mO}_s\dif s\Big)\cdot (\mV^\perp)^\top\cdot \vx}_{\text{orthogonal to the span of }\mV}.
\]

\subsection{Teacher-Student Training of OSDNet}

The OSDNet consists of two types of trainable parameters: \(\hat{\mO}_t\) and \(\hat{\vs}_t(\vx)\). When training OSDNet using the CFM loss, we show that these parameters can be decoupled and trained independently. Specifically, the OSDNet can be trained in a teacher-student framework, where \(\vv_t^*(\vx)\) serves as the teacher network, and \(\hat{\vv}_t(\vx)\) acts as the student network, which we present in \cref{prop:teacher_student_training_of_OSDNet}.

\begin{proposition}[Teacher-student training of OSDNet]
\label{prop:teacher_student_training_of_OSDNet}
Minimizing the CFM loss
\[
\mathcal{L}_{\text{CFM}}\big(\hat{\mO}_t, \hat{\vs}_t(\vx)\big) = \E_{t \sim \mathcal{U}[0,1], \vx_1 \sim p_1(\vx_1), \vx \sim p_t(\vx|\vx_1)} \big[ \big\lVert\hat{\vv}_t\big(\vx; \hat{\mO}_t, \hat{\vs}_t(\vx)\big) - \vu_t(\vx|\vx_1)\big\rVert_2^2 \big],
\]
is equivalent to minimizing the Teacher-Student Training (TST) loss
\[
\mathcal{L}_{\text{TST}}\big(\hat{\mO}_t, \hat{\vs}_t(\vx)\big) = \E_{t \sim \mathcal{U}[0,1], \vx \sim p_t(\vx)} \big[ \big\lVert\hat{\vv}_t\big(\vx; \hat{\mO}_t, \hat{\vs}_t(\vx)\big) - \vv_t^*(\vx)\big\rVert_2^2 \big],
\]
in the sense that \(\mathcal{L}_{\text{CFM}}\big(\hat{\mO}_t, \hat{\vs}_t(\vx)\big)\) and \(\mathcal{L}_{\text{TST}}\big(\hat{\mO}_t, \hat{\vs}_t(\vx)\big)\) differ only by a constant. As a result, training OSDNet reduces to the two independent minimization problems below:
\[
\min_{\hat{\mO}_t}\E_{t\sim\mathcal{U}[0,1], \vx\sim p_t(\vx)}\Big[\Big\lVert \mV^\perp \cdot\Big(\hat{\mO}_t +\dfrac{1}{1-t}\cdot \mI_{d-D}\Big)\cdot (\mV^\perp)^\top \cdot \vx\Big\rVert_2^2\Big]
\]
and
\[
\min_{\hat{\vs}_t(\vx)}\E_{t\sim\mathcal{U}[0,1], \vx\sim p_t(\vx)}\Big[\Big\lVert\mV\cdot\Big(\hat{\vs}_t(\vx)-\dfrac{1}{1 - t} \cdot \big( \mR \cdot \vw_t(\vx) - \mV^\top \cdot \vx \big)\Big)\Big\rVert_2^2\Big].
\]
\end{proposition}

\subsection{Decay of Off-Subspace Components}
\label{subsec:decay_of_off_subspace_components}

In this subsection, we examine the time-dependent matrix \(\hat{\mO}_t\), which governs the off-subspace components of generated samples through
\[
\mV^\perp\cdot \exp\Big(\int_0^1\hat{\mO}_s\dif s\Big)\cdot (\mV^\perp)^\top\cdot \vx.
\]
We start by simplifying the expectation term, transforming it into a function approximation problem in \cref{prop:teacher_student_training_of_O_t}

\begin{proposition}[Teacher-student training of \(\hat{\mO}_t\)]
\label{prop:teacher_student_training_of_O_t}
The solution to the minimization problem
\[
\min_{\hat{\mO}_t} \E_{t \sim \mathcal{U}[0,1],\ \vx \sim p_t(\vx)} \Big[ \Big\lVert \mV^\perp\cdot\Big( \hat{\mO}_t + \dfrac{1}{1 - t} \mI_{d-D} \Big) \cdot (\mV^\perp)^\top \cdot\vx \Big\rVert_2^2 \Big]
\]
reduces to solving the function approximation problem
\[
\min_{\hat{\mO}_t} \int_0^1 \big\lVert (1 - t)\cdot\hat{\mO}_t + \mI_{d-D} \big\rVert_{\mathrm{F}}^2 \dif t.
\]
\end{proposition}

Since our problem only depends on the temporal variable \( t \), we use the standard \emph{sinusoidal positional encoding} commonly adopted in Transformers \citep{vaswani2017attention} and U-Nets to embed temporal information into the velocity network. Specifically, we define the time embedding as:
\[
\emb(t) = \bigg( \sin\Big(\frac{s\cdot t}{\ell^{\frac{0}{\text{dim}}}}\Big), \cos\Big(\frac{s\cdot t}{\ell^{\frac{0}{\text{dim}}}}\Big), \sin\Big(\frac{s\cdot t}{\ell^{\frac{2}{\text{dim}}}}\Big), \cos\Big(\frac{s\cdot t}{\ell^{\frac{2}{\text{dim}}}}\Big), \dotsc, \sin\Big(\frac{s\cdot t}{\ell^{\frac{\text{dim}-2}{\text{dim}}}}\Big), \cos\Big(\frac{s\cdot t}{\ell^{\frac{\text{dim}-2}{\text{dim}}}}\Big) \bigg)^\top,
\]
where \(s\) denotes the scaling factor, \( \ell \) denotes the wavelength (typically set to \( \ell = 10000 \)), and ``\( \text{dim} \)'' is the dimension of the embedding.

Since \( \hat{\mO}_t \) is a diagonal matrix, the Frobenius norm reduces to the sum of the squares of its diagonal entries. We model each diagonal entry \( \hat{o}_t \) of \( \hat{\mO}_t \) as a linear transformation of the time embedding:
\[
\hat{o}_t = \vkappa^\top \emb(t),
\]
where \( \vkappa \in \mathbb{R}^{\text{dim}} \) is a parameter vector to be optimized. Our objective is to minimize the following loss function:
\[
\int_0^1 \big( (1 - t) \hat{o}_t + 1 \big)^2  \dif t = \int_0^1 \big( (1 - t) \vkappa^\top \emb(t) + 1 \big)^2  \dif t\coloneqq \vkappa^\top\mA\vkappa + 2\kappa^\top\vb + 1,
\]
where \(\mA =\int_0^1 (1 - t)^2 \emb(t) \emb(t)^\top \dif t \) and \(\vb = \int_0^1 (1 - t) \emb(t)\dif t\). Applying gradient descent to solve this quadratic problem, the gradient flow equation becomes
\[
\dfrac{\dif \vkappa}{\dif \tau} = -(2\mA\vkappa + 2\vb), \quad \tau\in[0, +\infty),
\]
which is a linear ODE with respect to \(\vkappa\) and can be solved explicitly. The general solution to this equation is given by:
\[
\vkappa(\tau) = \exp(-2\mA\tau)\cdot\vc - \mA^{-1}\vb,
\]
where \(\vc\) is a constant vector. Since \(\mA\) is a Gram matrix and is nonsingular, it is positive definite. Therefore, as \(\tau \rightarrow +\infty\), \(\vkappa(\tau)\) converges to \(- \mA^{-1}\vb\). 

Substituting \(\vkappa(\tau)\) into the expression for the off-space component, we have
\[
\begin{aligned}
\mV^\perp\cdot \exp\Big(\int_0^1\hat{\mO}_s\dif s\Big)\cdot (\mV^\perp)^\top\cdot \vx 
&= \exp\Big(\int_0^1\vkappa(\tau)^\top\emb(t)\dif t\Big)\cdot\mV^\perp\cdot (\mV^\perp)^\top\cdot \vx\\
&= \exp\Big(\int_0^1\big(\exp(-2\mA\tau)\cdot\vc - \mA^{-1}\vb\big)^\top\emb(t)\dif t\Big)\cdot\mV^\perp\cdot (\mV^\perp)^\top\cdot \vx\\
&= \exp\Big(\vc^\top\exp(-2\mA\tau)\ve-\vb^\top\mA^{-1}\ve\Big)\cdot\mV^\perp\cdot (\mV^\perp)^\top\cdot \vx,
\end{aligned}
\]
where \(\ve = \int_0^1\emb(t)\dif t\). Therefore, as \(\tau\rightarrow+\infty\), the expression decays to
\[
\exp\big(-\vb^\top\mA^{-1}\ve\big)\cdot\mV^\perp\cdot (\mV^\perp)^\top\cdot \vx.
\]

This limit has two implications:

\begin{itemize}
\item Even though the exponential term \(\exp\big(-\vb^\top\mA^{-1}\ve\big)\) may be small, it will never be exactly zero because \(\vb\), \(\mA^{-1}\), and \(\ve\) are all finite and nonsingular. Therefore, the contribution of the off-space component persists and does not vanish entirely.

\item The magnitude of this expression also depend on the projection of the initial point \(\vx\) onto the orthogonal complement of the space spanned by \(\mV\), i.e., \((\mV^\perp)^\top \cdot \vx\). This implies that the value of the limit will vary based on how much of the initial point lies in the subspace \(\mV^\perp\).
\end{itemize}

 We provide visualizations of the theoretical approximation results for \(\hat{\mO}_t\) and evidence of the decay of off-subspace components during actual training in \cref{app:numerical_experiments}.


\subsection{Generalization of Subspace Components}
\label{subsec:generalization_of_subspace_components}

In this subsection, we analyze the behavior of the generated samples within \(\range(\mV)\), i.e.,
\[
\mV \cdot \Big( \mV^\top \cdot \vx + \int_{0}^{1} \hat{\vs}_s(\vphi_s(\vx)) \dif s \Big).
\]
We begin by justifying why it is reasonable to assume that \(\vx \in \range(\mV)\), thereby avoiding the need to account for the off-subspace components of \(\vx\). This assumption is supported by demonstrating that \(\hat{\vs}_t(\vx) = \hat{\vs}_t(\mV \mV^\top \vx)\), implying that the subspace components of the generated samples are solely determined by the projection of \(\vx\) onto \(\range(\mV)\). With this assumption in place, the subspace component of the generated samples simplifies to the following form
\[
\vx + \int_{0}^{1} \mV \cdot \hat{\vs}_s(\vphi_s(\vx)) \dif s.
\]
This expression is, in fact, the integral solution to the ODE
\[
\dif \hat{\vphi}_t(\vx) = \mV\cdot\hat{\vs}_t(\hat{\vphi}_t(\vx))\dif t, \quad \hat{\vphi}_0(\vx)=\vx, \quad t\in[0, 1].
\]
Next, we model \(\hat{\vs}_t(\vx)\) as consisting of two parts. The first part is deterministic and is determined by factors such as the network architecture, network capacity, and other model characteristics. To avoid notational complexity, we will continue to denote this deterministic component as \(\hat{\vs}_t(\vx)\). The second part introduces stochasticity into the process, accounting for various sources of randomness such as truncation error, discretization error, and other perturbations inherent to the generation process. We model this stochastic component as being driven by a Gaussian process represented by \(\sigma \dif \vW_t\), where \(\vW_t\) denotes the standard Wiener process. Thus, we can model the generation process as an SDE, which takes the form
\[
\dif \hat{\vphi}_t(\vx) = \mV\cdot\hat{\vs}_t(\hat{\vphi}_t(\vx))\dif t + \sigma\dif\vW_t, \quad \hat{\vphi}_0(\vx)=\vx.
\]
In \cref{thm:generalization_of_subspace_components}, we compare the difference between the generated samples from two processes: (i) the optimal deterministic process, and (ii) the suboptimal stochastic process. \Cref{thm:generalization_of_subspace_components} shows that the difference between the generated samples can be broken down into three parts: the difference between the optimal and suboptimal vector fields, the noise introduced by the stochastic process, and the error caused by the squared noise term.

\begin{theorem}[Generalization of subspace components]
\label{thm:generalization_of_subspace_components}
Let \(\vphi_t^*(\vx)\) and \(\hat{\vphi}_t(\vx)\) be the solutions to the following ODE and SDE, respectively, defined on the interval \(t \in [0, 1 - \varepsilon]\), where \(0 < \varepsilon \leq 1\):
\[
\begin{aligned}
\dif \vphi_t^*(\vx) &= \mV \cdot \vs_t^*(\vphi_t^*(\vx))\dif t, \quad \vphi_0^*(\vx) = \vx, \\
\dif \hat{\vphi}_t(\vx) &= \mV \cdot \hat{\vs}_t(\hat{\vphi}_t(\vx))\dif t + \sigma\dif \vW_t, \quad \hat{\vphi}_0(\vx) = \vx,
\end{aligned}
\]
where \(\vW_t\) is a standard Wiener process in \(\mathbb{R}^d\), \(\mV\) is a matrix with orthonormal columns. Assume that the norms of the Jacobians and Hessians of \(\vs_t^*(\vx)\) are bounded:
\[
\|\nabla_\vx \vs_t^*(\vx)\|_{2} \leq L, \quad \|\nabla_\vx^2 \vs_t^*(\vx)\|_{\mathrm{F}} \leq M,
\]
for some constants \(L, M > 0\). Then, the difference between \(\vphi_{1 - \varepsilon}^*(\vx)\) and \(\hat{\vphi}_{1 - \varepsilon}(\vx)\) is given by
\[
\begin{aligned}
\vphi_{1 - \varepsilon}^*(\vx) - \hat{\vphi}_{1 - \varepsilon}(\vx) 
=& \int_0^{1 - \varepsilon} \nabla_\vx \vphi_{t, 1 - \varepsilon}^*(\hat{\vphi}_t(\vx)) \cdot \mV \cdot\big( \vs_t^*(\hat{\vphi}_t(\vx)) - \hat{\vs}_t(\hat{\vphi}_t(\vx)) \big) \dif t \\
& - \sigma\cdot\int_0^{1 - \varepsilon} \nabla_\vx \vphi_{t, 1 - \varepsilon}^*(\hat{\vphi}_t(\vx))\dif \vW_t 
- \dfrac{\sigma^2}{2}\cdot\int_0^{1 - \varepsilon}\sum_{i,j=1}^{d} \big(\nabla_\vx^2 \vphi_{t, 1 - \varepsilon}^*(\hat{\vphi}_t(\vx))\big)_{i,j,\colon} \dif t,
\end{aligned}
\]
where \(\vphi_{s, t}^*(\vx)\) denotes the solution to the ODE from time \(s\) to \(t\) with initial condition \(\vx\). Moreover, the expected squared difference satisfies the inequality
\[
\E[ \| \vphi_{1 - \varepsilon}^*(\vx) - \hat{\vphi}_{1 - \varepsilon}(\vx) \|_2^2 ] \leq C_1\cdot \int_0^{1 - \varepsilon} \E[ \| \vs_t^*(\hat{\vphi}_t(\vx)) - \hat{\vs}_t(\hat{\vphi}_t(\vx)) \|_2^2 ] \dif t + C_2\cdot\sigma^2 + C_3\cdot\sigma^4,
\]
for some constant \(C_1\), \(C_2\), and \(C_3\) depending on \(L\), \(M\), \(d\), and \(\varepsilon\). Here, the expectation is taken over the realizations of the standard Wiener process \(\vW_t\). 

Specifically, for \(\sigma = 0\), the inequality simplifies to the deterministic case
\[
\| \vphi_{1 - \varepsilon}^*(\vx) - \hat{\vphi}_{1 - \varepsilon}(\vx) \|_2^2 \leq C_4 \cdot \int_0^{1 - \varepsilon} \| \vs_t^*(\hat{\vphi}_t(\vx)) - \hat{\vs}_t(\hat{\vphi}_t(\vx)) \|_2^2  \dif t,
\]
where \(C_4\) is a constant depending only on \(L\) and \(\varepsilon\). 

Taking the expectation with respect to \(\vx \sim \mathcal{N}(\vzero, \mI_d)\), we obtain
\[
\E_{\vx \sim \mathcal{N}(\vzero, \mI_d)}[ \| \vphi_{1 - \varepsilon}^*(\vx) - \hat{\vphi}_{1 - \varepsilon}(\vx) \|_2^2 ] \leq C_5 \cdot \E_{t \sim \mathcal{U}(0, 1), \vx \sim p_t}[ \| \vs_t^*(\vx) - \hat{\vs}_t(\vx) \|_2^2 ],
\]
where \(p_t\) is the distribution of \(\vphi_t^*(\vx)\) and \(C_5\) is a constant depending only on the Lipschitz constant of \(\hat{\vs}_t(\vx)\) and \(\varepsilon\).
\end{theorem}

We provide several clarifications regarding the assumptions and results of \cref{thm:generalization_of_subspace_components}.

\begin{itemize}
\item The reason for truncating the interval \([0, 1]\) to \([0, 1 - \varepsilon]\) is that \(\vs_t^*(\vx)\) becomes singular at \(t = 1\), leading to unbounded norms for its Jacobians and Hessians at \(t = 1\).
\item As \(\varepsilon \to 0\), we obtain \(\vphi_{1-\varepsilon}^*(\vx) = \vy^i\) for some \(1 \leq i \leq N\), as shown in \cref{thm:memorization}. The last inequality indicates that the extent to which generated samples using the suboptimal velocity field will deviate from the real data points is upper bounded by the training loss of \(\hat{\vs}_t(\vx)\).
\item A similar inequality to the last inequality was also derived in \citep[Theorem 1]{benton2024error}, with the difference that they use a time-varying Lipschitz constant for the suboptimal velocity field, while we use a universal constant.
\end{itemize}

We provide visualizations of the generated samples and their progression during training, followed by an analysis of the relationship between MSE and training loss in \cref{app:numerical_experiments}.

\section{Conclusion}

In this paper, we provided theoretical insights into the challenges of capturing the sample data subspace using generative models. By leveraging Flow Matching models and treating the real data distribution as discrete, we derived analytical expressions for the optimal velocity field and analyzed the geometry of generation paths, demonstrating that generated samples memorize real data points and accurately represent the underlying subspace. To generalize beyond the optimal case, we introduced the Orthogonal Subspace Decomposition Network (OSDNet), which systematically decomposes the velocity field into subspace and off-subspace components. Our analysis showed that while the off-subspace component decays, the subspace component generalizes within the sample data subspace, with deviations from the real data points influenced by the training error of the OSDNet. Looking ahead, these theoretical results can guide the design of algorithms to improve the generalization capabilities of Flow Matching models while maintaining the fidelity of generated samples, which we leave as future work.

\bibliographystyle{abbrvnat}
\bibliography{auxiliaries/references}

\clearpage
\paragraph{Roadmap.} The appendix is organized as follows:
\begin{itemize}
\item \Cref{app:addition_literature_review} provides an in-depth literature review on generative models, along with a comparative analysis of the most relevant works.
\item \Cref{app:flow_matching_models} offers a detailed background on Flow Matching models.
\item \Cref{app:time_ordered_exponential} explains the background on time-ordered matrix exponentials.
\item \Cref{app:proofs_to_theorems} contains the proofs of all the theorems and propositions.
\item \Cref{app:numerical_experiments} presents the details and results of the numerical experiments.
\end{itemize}

\appendix
\section{Additional Literature Review}
\label{app:addition_literature_review}

\paragraph{Generative models.} Generative models have gained significant attention in recent years for their capacity to learn intricate data distributions and generate novel samples that closely resemble the training data. In particular, they have significantly advanced the field of image generation, with various models emerging over time, each contributing uniquely to the landscape. 

In the past decade, the journey began with \emph{Variational Autoencoders (VAEs)} \citep{rezende2014stochastic, kingma2014auto, chen2017variational, child2021very}, which introduce a probabilistic approach by mapping input data to a latent space and reconstructing it, providing a smooth latent space ideal for interpolation, though often producing blurry images due to the averaging effect of their reconstruction loss.

Following VAEs, \emph{Generative Adversarial Networks (GANs)} \citep{goodfellow2014generative, radford2016unsupervised, arjovsky2017wasserstein, nguyen2017dual, ghosh2018multi, lin2018pacgan, brock2019large, karras2020training} revolutionized image generation by employing a game-theoretic framework involving a generator and a discriminator. This adversarial setup results in sharp and high-quality images but comes with challenges like training instability and mode collapse, making GANs difficult to optimize effectively.

\emph{Autoregressive Models} \citep{van2016pixel} took a different approach by modeling the distribution of data as a product of conditional distributions, generating each pixel or data point sequentially. These models excel in producing high-fidelity images but suffer from slow sampling speeds due to their sequential nature.

\emph{Normalizing Flows} \citep{dinh2017density, kingma2018glow, grathwohl2019scalable} emerged as a powerful tool by transforming simple probability distributions into complex ones through invertible mappings, allowing for exact likelihood computation and flexible modeling of data distributions. Despite their theoretical elegance, normalizing flows can be computationally demanding, particularly in computing the Jacobian determinant during training.

\emph{Energy-Based Models (EBMs)} \citep{du2019implicit, gao2021learning} contributed to generative modeling by defining an energy function over data, where lower energy corresponds to more likely data. While EBMs offer a powerful framework for modeling complex distributions, they can be computationally intensive and require careful handling to ensure stability and convergence.

\emph{Diffusion Models} \citep{hyvarinen2005estimation, vincent2011connection, song2019generative,ho2020denoising,song2021score,karras2022elucidating} introduced a novel approach by defining a process that gradually transforms data into noise and then learns to reverse this process to generate new data samples. These models have shown remarkable results in generating high-quality images and offer stability advantages over GANs, albeit often requiring longer generation time.

Most recently, \emph{Flow Matching Models} \citep{lipman2023flow,liu2023flow} have been proposed, combining ideas from normalizing flows and diffusion models to create a framework that matches data distributions through a series of transformations. These models aim to leverage the strengths of both approaches, offering robust and efficient data generation capabilities.

\paragraph{Comparison with \citep{chen2023score}.} \citet{chen2023score} demonstrate that generated data distribution of diffusion models converges to a close vicinity of the target distribution by providing convergence rates for score approximation and distribution recovery, quantified using \(L_2\) error, total variation distance, and \(2\)-Wasserstein distance. While our work is closely related, there are notable differences in assumptions, methodology, and the focus of our analysis. (\romannumeral 1) Regarding assumptions about the data distribution, \citet{chen2023score} assume a continuous distribution, specifically representing the data as \(\vx = \mA\vz\) with latent variable \(\vz \sim p_{\mathbf{z}}\) and \(p_{\vz} > 0\). In contrast, we consider a discrete data distribution, which aligns more closely with the practical scenario where only a finite number of samples are available. (\romannumeral 2) In terms of the generative model, \citet{chen2023score} center their study on diffusion models, where the generative process is described by a reverse SDE. In our work, we focus on Flow Matching models, using ODE to describe the generative process, which allows us to derive explicit analytical expressions for the velocity field. (\romannumeral 3) With respect to network architecture, \citet{chen2023score} employ a ReLU-based network to jointly process spatial and temporal inputs \((\vx, t)\). Conversely, we treat spatial and temporal components separately, using sinusoidal positional encoding to embed temporal information, which enables more effective modeling of temporal dependencies. (\romannumeral 4) \citet{chen2023score} emphasize providing quantitative convergence results for score approximation and distribution recovery. Our work, in contrast, highlights specific behaviors of the generative process, particularly memorization under the optimal velocity field and generalization under the suboptimal velocity field.

\paragraph{Comparison with \citep{wang2024diffusion}.} \citet{wang2024diffusion} connect the optimization of diffusion model training loss with solving the canonical subspace clustering problem for the training data. They demonstrate that, when the number of samples exceeds the intrinsic dimension of the subspaces, the optimal training solutions can recover the underlying data distribution. While our research shares similarities, there are key differences in terms of assumptions, methodology, and the focus of our analysis. (\romannumeral 1) In terms of data distribution assumptions, \citet{wang2024diffusion} model the data as a mixture of low-rank Gaussians, specifically, \(\vx \sim \sum_{k=1}^{K} \pi_k \mathcal{N}(\vzero, \mU_k \mU_k^*)\). By contrast, we focus on a discrete data distribution, which offers more flexibility in modeling low-dimensional distributions. (\romannumeral 2) For the generative model, \citet{wang2024diffusion} focus on diffusion models, whereas we emphasize Flow Matching models. (\romannumeral 3) Regarding network architecture, \citet{wang2024diffusion} employ an explicit low-rank denoising autoencoder, while our work explores more general architectures. (\romannumeral 4) \citet{wang2024diffusion} investigate when and why diffusion models can recover the underlying data distribution without encountering the curse of dimensionality, whereas our work focuses on understanding the generative process, with particular attention to memorization under the optimal velocity field and generalization under the suboptimal velocity field.

\section{Flow Matching Models}
\label{app:flow_matching_models}

Flow Matching (FM) \citep{lipman2023flow} provides a flexible and powerful framework for transforming a simple prior distribution \( p_0(\vx) \), typically Gaussian, into a more complex target distribution \( p_1(\vx) \) by learning a velocity field. What sets FM apart is its ability to leverage a variety of probability paths to guide this transformation, extending beyond the typical diffusion-based approaches.

The process of transformation is driven by a time-dependent velocity field \( \vu_t(\vx) \), which defines a probability path \( p_t(\vx) \) for \( t \in [0,1] \). The evolution of the flow \(\vphi_t(\vx)\) over time is governed by the ODE
\[
\frac{\dif \vphi_t(\vx)}{\dif t} = \vu_t(\vphi_t(\vx)), \quad \vphi_0(\vx) = \vx, \quad t\in [0, 1].
\]
As the points \( \vphi_t(\vx) \) flow along this velocity field, the probability density function \( p_t(\vx) \) evolves according to the continuity equation (a.k.a., transport equation), which ensures the conservation of probability mass along the flow:
\[
\frac{\partial p_t(\vx)}{\partial t} + \nabla \cdot \big( p_t(\vx) \vu_t(\vx) \big) = \vzero, \quad t\in [0, 1].
\]
Since directly computing \( \vu_t(\vx) \) is generally intractable, FM approximates it using the neural network \( \vv_t(\vx; \vtheta) \), where \(\vtheta\) denotes the network parameters. The network is trained to minimize the Flow Matching loss
\[
\mathcal{L}_{\text{FM}}(\vtheta) = \E_{t \sim \mathcal{U}[0,1], \vx \sim p_t(\vx)} \big[ \|\vv_t(\vx; \vtheta) - \vu_t(\vx)\|_2^2 \big].
\]
However, in practical settings, the exact form of \( p_1(\vx) \) is unknown, and only samples from \( p_1(\vx) \) are available. To address this, FM introduces Conditional Flow Matching (CFM), which constructs conditional probability paths \( p_t(\vx|\vx_1) \) for each real data point \( \vx_1 \sim p_1(\vx_1) \). The conditional velocity field \( \vu_t(\vx|\vx_1) \) is designed to transport the initial distribution \( p_0(\vx) \) to \( p_1(\vx|\vx_1) \), leading to the CFM loss:
\[
\mathcal{L}_{\text{CFM}}(\vtheta) = \E_{t \sim \mathcal{U}[0,1], \vx_1 \sim p_1(\vx_1), \vx \sim p_t(\vx|\vx_1)} \big[ \|\vv_t(\vx; \vtheta) - \vu_t(\vx|\vx_1)\|_2^2 \big].
\]
FM is flexible in its choice of probability paths. To set the stage for more specific discussion, we assume that
\[
p_t(\vx|\vx_1) = \mathcal{N}(\vx; \vmu_t(\vx_1), \sigma_t(\vx_1)^2\mI),
\]
where \(\vmu\colon [0,1] \times \mathbb{R}^d \to \mathbb{R}^d\) is the time-dependent mean of the Gaussian distribution, and \(\sigma\colon [0,1] \times \mathbb{R} \to \mathbb{R}_{>0}\) describes a time-dependent scalar standard deviation. We further assume that when conditioned on \(\vx_1\), the flow \(\vphi_t\) is given by
\[
\vphi_t(\vx) = \sigma_t(\vx_1)\vx + \mu_t(\vx_1).
\]
By substituting the expression of \(\vphi_t\) into the flow equation, we derive that
\[
\vu_t(\vx|\vx_1) = \dfrac{\sigma_t'(\vx_1)}{\sigma_t(\vx_1)}\cdot\big(\vx-\vmu_t(\vx_1)\big) + \vmu_t'(\vx_1).
\]
We now explore two specific instances of conditional probability paths --- variance-preserving (VP) diffusion paths and optimal transport (OT) paths. For VP diffusion paths, the noise level increases over time while preserving the variance of the data. Specifically, we choose
\[
\vmu_t(\vx_1) = \alpha_{1-t}\vx_1, \quad \sigma_t^2(\vx_1) = 1 - \alpha_{1-t}^2,\quad \alpha_t=\exp\Big(-\dfrac{1}{2}T(t)\Big), \quad T(t)=\int_{0}^{t}\beta(s)\dif s,
\]
where \( \beta \) is the noise scale function. The corresponding velocity field \( \vu_t \) is
\[
\vu_t(\vx|\vx_1) = \frac{\alpha'_{1-t}}{1 - \alpha_{1-t}^2} \cdot( \alpha_{1-t} \vx - \vx_1 ) = - \frac{T'(1-t)}{2}\cdot \frac{\exp(-T(1-t))\vx - \exp(-T(1-t)/2)\vx_1}{1 - \exp(-T(1-t))}.
\]
For OT paths, they involve a linear interpolation of both the mean and the standard deviation, leading to straight-line trajectories in the space of probability measures. Specifically, we define
\[
\vmu_t(\vx_1) = t\vx_1, \quad \sigma_t(\vx_1) = 1 - (1 - \sigma_{\min})t,
\]
where \( \sigma_{\min} \) is a small constant. The resulting velocity field \( \vu_t \) is
\[
\vu_t(\vx|\vx_1) = \frac{\vx_1 - (1 - \sigma_{\min})\vx}{1 - (1 - \sigma_{\min})t}.
\]

\section{Time-Ordered Exponential}
\label{app:time_ordered_exponential}

The time-ordered exponential, denoted by \(\mathcal{T} \exp\), is a key concept in solving linear differential equations with time-dependent coefficients, especially when the involved matrices do not commute at different times. In contrast to the standard matrix exponential, which applies to constant matrices, the time-ordered exponential generalizes this concept to account for the non-commutative nature of time-varying operators. It ensures that the sequence of matrix multiplications reflects the chronological order of time, thus preserving the correct causal structure in time-dependent problems.

To understand the necessity of the time-ordered exponential, consider the following linear differential equation
\[
\frac{\dif \vphi_t(\vx)}{\dif t} = \mA(t) \vphi_t(\vx), \quad \vphi_{t_0}(\vx) = \vx,\quad t\in[t_0, T],
\]
where \(\mA(t)\) is a time-dependent square matrix. If \(\mA(t)\) were constant, the solution to this equation would simply be the matrix exponential
\[
\vphi_t(\vx) = \exp\big( (t - t_0) \mA \big) \vx.
\]
However, when \(\mA(t)\) varies with time and does not commute at different instances (i.e., \(\mA(t_1)\cdot\mA(t_2) \neq \mA(t_2)\cdot\mA(t_1)\) for \(t_1 \neq t_2\)), the standard matrix exponential cannot be directly applied. In such cases, the time-ordered exponential must be used to incorporate the non-commutativity of the time-dependent operators.

The time-ordered exponential is formally defined as:
\[
\mathcal{T} \exp\Big( \int_{t_0}^{t} \mA(\tau) \dif \tau \Big) = \mI + \int_{t_0}^{t} \mA(\tau_1) \dif \tau_1 + \int_{t_0}^{t} \mA(\tau_1) \Big( \int_{t_0}^{\tau_1} \mA(\tau_2) \dif \tau_2 \Big) \dif \tau_1 + \cdots.
\]
This series can be written more compactly using the Dyson series
\[
\mathcal{T} \exp\Big( \int_{t_0}^{t} \mA(\tau) \dif \tau \Big) = \sum_{n=0}^{\infty} \int_{t_0}^{t} \dif \tau_n \int_{t_0}^{\tau_n} \dif \tau_{n-1} \cdots \int_{t_0}^{\tau_2} \dif \tau_1 \, \mA(\tau_n) \mA(\tau_{n-1}) \cdots \mA(\tau_1).
\]
In this expression, the time-ordering operator \(\mathcal{T}\) ensures that the matrices are multiplied in a time-ordered manner, with the later time matrices appearing to the left of earlier ones. This accounts for the non-commutative behavior of the matrices over time.

The time-ordered exponential provides the solution to the original time-dependent differential equation as follows
\[
\vphi_t(\vx) = \mathcal{T} \exp\Big( \int_{t_0}^{t} \mA(\tau) \dif \tau \Big) \vx.
\]
This solution can be derived iteratively starting from the integral form of the differential equation:
\[
\vphi_t(\vx) = \vx + \int_{t_0}^{t} \mA(\tau_1) \vphi_{\tau_1}(\vx) \dif \tau_1.
\]
By substituting the expression for \(\vphi_{\tau_1}(\vx)\) back into itself, we obtain the following recursive relation:
\[
\begin{aligned}
\vphi_t(\vx) &= \vx + \int_{t_0}^{t} \mA(\tau_1) \Big( \vx + \int_{t_0}^{\tau_1} \mA(\tau_2) \vphi_{\tau_2}(\vx) \dif \tau_2 \Big) \dif \tau_1 \\
&= \vx + \int_{t_0}^{t} \mA(\tau_1) \vx \dif \tau_1 + \int_{t_0}^{t} \mA(\tau_1) \Big( \int_{t_0}^{\tau_1} \mA(\tau_2) \vphi_{\tau_2}(\vx) \dif \tau_2 \Big) \dif \tau_1 \\
&= \vx + \int_{t_0}^{t} \mA(\tau_1) \vx \dif \tau_1 + \int_{t_0}^{t} \mA(\tau_1) \Big( \int_{t_0}^{\tau_1} \mA(\tau_2) \vx \dif \tau_2 \Big) \dif \tau_1 + \cdots,
\end{aligned}
\]
which generates the Dyson series representation.

\section{Proofs to Theorems}
\label{app:proofs_to_theorems}

\subsection{Derivation of the Optimal Velocity Field}

\begin{proposition}[Optimal velocity field for conditional flow matching]
Assume that the function class \(\{\vv_t(\vx;\vtheta)\}\) has enough capacity. Then the optimal velocity field \( \vv_t^*(\vx) \) which minimizes the Conditional Flow Matching (CFM) loss
\[
\mathcal{L}_{\text{CFM}}(\vtheta) = \E_{t \sim \mathcal{U}[0,1], \vx_1 \sim p_1(\vx_1), \vx \sim p_t(\vx|\vx_1)} \big[ \|\vv_t(\vx; \vtheta) - \vu_t(\vx|\vx_1)\|_2^2 \big],
\]
is given by
\[
\vv_t^*(\vx) = \E_{\vx_1 \sim p_1(\vx_1)}[ \vu_t(\vx|\vx_1)|\vx, t].
\]
\end{proposition}

\begin{proof}
Since the function class \(\{\vv_t(\vx;\vtheta)\}\) has enough capacity, we aim to find \( \vv_t^*(\vx) \) such that
\[
\vv_t^*(\vx) = \arg\min_{\vv_t(\vx)} \mathcal{L}_{\text{CFM}}.
\]
Because \( \vv_t(\vx; \vtheta) \) depends on \( \vx \) and \( t \), we can consider the loss function pointwise with respect to \( \vx \) and \(t\):
\[
\mathcal{L}(\vx, t, \vv_t(\vx)) = \E_{\vx_1}\big[\| \vv_t(\vx) - \vu_t(\vx|\vx_1)\|_2^2\big|\vx, t\big],
\]
where the expectation is over \( \vx_1 \sim p_1(\vx_1) \), conditioned on fixed \( \vx \) and \(t\).

To find the \( \vv_t^*(\vx) \) that minimizes \( \mathcal{L}(\vx, t, \vv_t(\vx)) \), we can set the gradient of \( \mathcal{L} \) with respect to \( \vv_t(\vx) \) to zero, i.e., 
\[
\nabla_{\vv_t(\vx)} \mathcal{L}(\vx, t, \vv_t(\vx)) = \vzero.
\]
By definition, we have
\[
\nabla_{\vv_t(\vx)} \mathcal{L}(\vx, t, \vv_t(\vx)) = \nabla_{\vv_t(\vx)} \E_{\vx_1}\big[\| \vv_t(\vx) - \vu_t(\vx|\vx_1)\|_2^2\big|\vx, t\big].
\]
Since the expectation operator is linear, we can interchange the gradient and expectation and obtain
\[
\nabla_{\vv_t(\vx)} \mathcal{L}(\vx, t, \vv_t(\vx)) = 2\cdot \E_{\vx_1} \big[\vv_t(\vx) - \vu_t(\vx|\vx_1)\big|\vx, t\big].
\]
Setting the gradient to zero and simplifying, we have
\[
\vv_t^*(\vx) = \E_{\vx_1}[ \vu_t(\vx|\vx_1)|\vx, t].
\]
Therefore, the optimal velocity field \( \vv_t^*(\vx) \) is the conditional expectation of \( \vu_t(\vx|\vx_1) \) given \( \vx \) and \(t\), i.e.,
\[
\vv_t^*(\vx) = \E_{\vx_1 \sim p_1(\vx_1)}[ \vu_t(\vx|\vx_1)|\vx, t].
\]
This completes the proof.
\end{proof}

\begin{theorem}[Optimal velocity field for discrete target distribution]
Let \( p_0 \sim \mathcal{N}(\vzero, \mI_d) \), and suppose \( p_1 \) is a discrete distribution over a set of points \( \{\vy^i\colon 1\leq i\leq N\}\subset \mathbb{R}^d \), given by
\[
p_1 \sim \frac{1}{N} \sum_{i=1}^{N} \delta_{\vy^i},
\]
where \( \delta_{\vy^i} \) denotes the Dirac delta function centered at \( \vy^i \). Assume that the conditional velocity field is
\[
\vu_t(\vx|\vx_1) = \dfrac{\sigma_t'(\vx_1)}{\sigma_t(\vx_1)}\cdot\big(\vx-\vmu_t(\vx_1)\big) + \vmu_t'(\vx_1).
\]
Then the optimal velocity field \( \vv_t^*(\vx) \) that minimizes the CFM loss is given by
\[
\vv_t^*(\vx) = \sum_{i=1}^{N} \Big(\dfrac{\sigma_t'(\vy^i)}{\sigma_t(\vy^i)}\cdot\big(\vx-\vmu_t(\vy^i)\big) + \vmu_t'(\vy^i)\Big)\cdot\dfrac{\exp\big(-\|(\vx-\vmu_t(\vy^i))/\sigma_t(\vy^i)\|_2^2/2\big)}{\sum_{j=1}^{N}\exp\big(-\|(\vx-\vmu_t(\vy^j))/\sigma_t(\vy^j)\|_2^2/2\big)}.
\]
Specifically, in the case of OT paths where \(\vmu_t(\vx_1)=t\vx_1\) and \(\sigma_t(\vx)=1-t\), the optimal velocity field simplifies to
\[
\vv_t^*(\vx) = \sum_{i=1}^{N}\dfrac{\vy^i-\vx}{1-t}\cdot\dfrac{\exp\big(-\|(\vx-t\vy^i)/(1-t)\|_2^2/2\big)}{\sum_{j=1}^{N}\exp\big(-\|(\vx-t\vy^j)/(1-t)\|_2^2/2\big)}. 
\]
\end{theorem}

\begin{proof}
We begin by computing the conditional probability
\[
\mathbb{P}(\vx_1=\vy^i|\sigma_t(\vx_1)\vx_0+\vmu_t(\vx_1)=\vx),
\]
where \(\vx_0\sim p_0\). By applying Bayes' theorem, we have
\[
\begin{aligned}
\mathbb{P}(\vx_1=\vy^i|\sigma_t(\vx_1)\vx_0+\vmu_t(\vx_1)=\vx)
&= \dfrac{\mathbb{P}(\sigma_t(\vx_1)\vx_0+\vmu_t(\vx_1)=\vx|\vx_1=\vy^i)\cdot\mathbb{P}(\vx_1=\vy^i)}{\mathbb{P}(\sigma_t(\vx_1)\vx_0+\vmu_t(\vx_1)=\vx)}\\
&= \dfrac{\mathbb{P}(\sigma_t(\vx_1)\vx_0+\vmu_t(\vx_1)=\vx|\vx_1=\vy^i)\cdot\mathbb{P}(\vx_1=\vy^i)}{\sum_{j=1}^{N}\mathbb{P}(\sigma_t(\vx_1)\vx_0+\vmu_t(\vx_1)=\vx|\vx_1=\vy^j)\cdot\mathbb{P}(\vx_1=\vy^j)}\\
&= \dfrac{\mathcal{N}\big((\vx-\vmu_t(\vy^i))/\sigma_t(\vy^i);\vzero, \mI_d\big)}{\sum_{j=1}^{N}\mathcal{N}\big((\vx-\vmu_t(\vy^j))/\sigma_t(\vy^j);\vzero, \mI_d\big)}\\
&= \dfrac{\exp\big(-\|(\vx-\vmu_t(\vy^i))/\sigma_t(\vy^i)\|_2^2/2\big)}{\sum_{j=1}^{N}\exp\big(-\|(\vx-\vmu_t(\vy^j))/\sigma_t(\vy^j)\|_2^2/2\big)}.
\end{aligned}
\]
Using this, we can now compute the conditional expectation to obtain the optimal velocity field. By the definition of conditional expectation, we have
\[
\vv_t^*(\vx) = \sum_{i=1}^{N} \Big(\dfrac{\sigma_t'(\vy^i)}{\sigma_t(\vy^i)}\cdot\big(\vx-\vmu_t(\vy^i)\big) + \vmu_t'(\vy^i)\Big)\cdot\dfrac{\exp\big(-\|(\vx-\vmu_t(\vy^i))/\sigma_t(\vy^i)\|_2^2/2\big)}{\sum_{j=1}^{N}\exp\big(-\|(\vx-\vmu_t(\vy^j))/\sigma_t(\vy^j)\|_2^2/2\big)}.
\]
Now, consider the specific case of OT paths, where \( \vmu_t(\vx) = t\vx_1 \) and \( \sigma_t(\vx) = 1 - t \). Substituting these into our expression for the optimal velocity field, we obtain the simplified form
\[
\vv_t^*(\vx) = \sum_{i=1}^{N}\dfrac{\vy^i-\vx}{1-t}\cdot\dfrac{\exp\big(-\|(\vx-t\vy^i)/(1-t)\|_2^2/2\big)}{\sum_{j=1}^{N}\exp\big(-\|(\vx-t\vy^j)/(1-t)\|_2^2/2\big)}.
\]
This finalizes the proof.
\end{proof}

\subsection{Path Geometry under the Optimal Velocity Field}

\begin{theorem}[Probability bound for softmax weight concentration]
Let \( \{\vy^i\colon 1\leq i\leq N\}\subset \mathbb{R}^d \) be real data points such that \( \|\vy^i - \vy^j\|_2 \geq M \) for all \( 1\leq i\neq j\leq N \). Consider the softmax weight function:
\[
\vw_t(\vx) = \softmax\Big(-\dfrac{1}{2}\Big\Vert\dfrac{\vx - t\vy^1}{1 - t}\Big\Vert_2^2, -\dfrac{1}{2}\Big\Vert\dfrac{\vx - t\vy^2}{1 - t}\Big\Vert_2^2, \dotsc, -\dfrac{1}{2}\Big\Vert\dfrac{\vx - t\vy^N}{1 - t}\Big\Vert_2^2\Big).
\]
Assume that
\[
\vx \sim p_t = \dfrac{1}{N}\sum_{i=1}^{N} \mathcal{N}(t\vy^i, (1 - t)^2\mI_d).
\]
Then, the probability that \(\vx\) falls into the region \(\mathcal{R} \coloneqq \{\vx\colon \max(\vw_t(\vx)) \leq \tau < 1\}\), i.e., where the softmax weight vector is not close to a one-hot vector in the \(\tau\)-sense, is bounded by
\[
\P(\vx\in\mathcal{R})\leq
\dfrac{1}{\sqrt{2\pi}}\cdot\dfrac{1-t}{t}\cdot\dfrac{1}{M}\cdot\log\dfrac{\tau(N-1)}{1-\tau}\cdot N(N-1).
\]
\end{theorem}

\begin{proof}
We begin by noting that the softmax function is invariant under translations; that is, adding the same constant to each input does not change the output. Thus, for
\[
\vw_t(\vx) = \softmax\Big(-\dfrac{1}{2}\Big\Vert\dfrac{\vx - t\vy^1}{1 - t}\Big\Vert_2^2, -\dfrac{1}{2}\Big\Vert\dfrac{\vx - t\vy^2}{1 - t}\Big\Vert_2^2, \dotsc, -\dfrac{1}{2}\Big\Vert\dfrac{\vx - t\vy^N}{1 - t}\Big\Vert_2^2\Big),
\]
to be close to a one-hot vector, the largest entry of the vector must be significantly larger than the second-largest entry. 
Conversely, if \(\max(\vw_t(\vx)) \leq \tau < 1\), then the difference between the largest and the second-largest inputs to the softmax function is bounded above. Specifically, for the softmax function \(\softmax(z_1, z_2, \dotsc, z_N)\), let \(z_{\max}\) and \(z_{\text{next}}\) denote the largest and second-largest entries, respectively. If
\[
\frac{\exp(z_{\max})}{\sum_{k=1}^N \exp(z_k)} \leq \tau,
\]
then
\[
\exp(z_{\max}) \leq \tau \sum_{k=1}^N \exp(z_k) \leq \tau \big(\exp(z_{\max}) + (N - 1)\cdot\exp(z_{\text{next}})\big).
\]
Rearranging this inequality gives
\[
(1 - \tau)\cdot\exp(z_{\max}) \leq \tau (N - 1)\cdot\exp(z_{\text{next}}),
\]
which implies that
\[
z_{\max} - z_{\text{next}} \leq \log\dfrac{\tau(N-1)}{1-\tau}.
\]
Next, define the region \( \mathcal{R}_{i,j} \) for each pair \(i \neq j\) as
\[
\mathcal{R}_{i,j} \coloneqq \Big\{\vx \colon 0 \leq -\dfrac{1}{2}\Big\Vert\dfrac{\vx - t\vy^i}{1 - t}\Big\Vert_2^2 + \dfrac{1}{2}\Big\Vert\dfrac{\vx - t\vy^j}{1 - t}\Big\Vert_2^2 \leq \log\dfrac{\tau(N-1)}{1-\tau} \Big\}.
\]
This region consists of points \(\vx\) where the difference between the inputs corresponding to \(\vy^i\) and \(\vy^j\) is small. Since the maximum over \(i\) of \(\vw_t(\vx)_i\) being less than or equal to \(\tau\) implies that there exists some \(i \neq j\) such that the difference between the corresponding inputs is bounded, we have
\[
\mathcal{R} \subset \bigcup_{i \neq j}^{N} \mathcal{R}_{i,j}.
\]
We now estimate the probability that a Gaussian mixture random variable falls into \( \mathcal{R}_{i,j} \), i.e., 
\[
\mathbb{P}(\vx \in \mathcal{R}_{i,j}), \quad \vx \sim p_t = \dfrac{1}{N}\sum_{i=1}^{N} \mathcal{N}(t\vy^i, (1 - t)^2\mI_d).
\]
By the definition of \( \mathcal{R}_{i,j} \),
\[
\begin{aligned}
&\mathbb{P}(\vx \in \mathcal{R}_{i,j})\\
=\,& \mathbb{P}\Big(\Big\{\vx \colon 0 \leq -\dfrac{1}{2}\Big\Vert\dfrac{\vx - t\vy^i}{1 - t}\Big\Vert_2^2 + \dfrac{1}{2}\Big\Vert\dfrac{\vx - t\vy^j}{1 - t}\Big\Vert_2^2 \leq \log\dfrac{\tau(N-1)}{1-\tau}\Big\}\Big) \\
=\,& \mathbb{P}\Big(\Big\{\vx \colon 0 \leq \|\vx - t\vy^j\|_2^2 - \|\vx - t\vy^i\|_2^2 \leq 2(1-t)^2 \log \dfrac{\tau(N-1)}{1-\tau} \Big\}\Big) \\
=\,& \mathbb{P}\Big(\Big\{\vx \colon \dfrac{t}{2} \cdot \big(\|\vy^j\|_2^2 - \|\vy^i\|_2^2\big) \leq \langle \vx, \vy^i - \vy^j \rangle \leq \dfrac{1}{2t} \cdot \big(2(1-t)^2 \log \dfrac{\tau(N-1)}{1-\tau} + t^2 (\|\vy^j\|_2^2 - \|\vy^i\|_2^2)\big)\Big\}\Big) \\
\leq\,& \dfrac{1}{\sqrt{2\pi}} \cdot \dfrac{1-t}{t} \cdot \dfrac{1}{\|\vy^j - \vy^i\|_2} \cdot \log \dfrac{\tau(N-1)}{1-\tau} \\
\leq\,& \dfrac{1}{\sqrt{2\pi}} \cdot \dfrac{1-t}{t} \cdot \dfrac{1}{M} \cdot \log \dfrac{\tau(N-1)}{1-\tau}.
\end{aligned}
\]
In the last but two inequality, we use the fact that if \( \vz \sim \mathcal{N}(t\vy^k, (1-t)^2\mI_d) \), then
\[
\langle \vz, \vy^i - \vy^j \rangle \sim \mathcal{N}\big(t \cdot (\vy^i - \vy^j)^\top \vy^k, (1-t)^2 \cdot \|\vy^i - \vy^j\|_2^2\big).
\]
As a result, for any \( A \leq B \), we have
\[
\begin{aligned}
&\mathbb{P}\big(A \leq \langle \vz, \vy^i - \vy^j \rangle \leq B\big)\\
=\,& \mathbb{P}\Big(\dfrac{A - t \cdot (\vy^i - \vy^j)^\top \vy^k}{(1-t) \cdot \|\vy^i - \vy^j\|_2} \leq \dfrac{\langle \vz, \vy^i - \vy^j \rangle - t \cdot (\vy^i - \vy^j)^\top \vy^k}{(1-t) \cdot \|\vy^i - \vy^j\|_2} \leq \dfrac{B - t \cdot (\vy^i - \vy^j)^\top \vy^k}{(1-t) \cdot \|\vy^i - \vy^j\|_2}\Big) \\
\leq\,& \dfrac{1}{\sqrt{2\pi}} \cdot \dfrac{B-A}{(1-t) \cdot \|\vy^i - \vy^j\|_2},
\end{aligned}
\]
where we use the fact that the maximum value of the standard normal PDF is \(1/\sqrt{2\pi}\). Using this inequality for each component of the Gaussian mixture yields the desired result. Finally, applying Boole's inequality, we have
\[
\mathbb{P}(\vx \in \mathcal{R}) \leq \sum_{i \neq j}^{N} \mathbb{P}(\vx \in \mathcal{R}_{i,j}) \leq \dfrac{1}{\sqrt{2\pi}} \cdot \dfrac{1-t}{t} \cdot \dfrac{1}{M} \cdot \log \dfrac{\tau(N-1)}{1-\tau} \cdot N(N-1),
\]
which completes the proof.
\end{proof}

\begin{theorem}[Hierarchy emergence of generation paths]
Let \(C_{i, j}\;(i \in \mathcal{I},\, j \in \mathcal{J}_i)\) be a collection of disjoint, bounded, closed convex sets where the distribution \(p_1\) is supported. Define
\[
C_i = \conv\Big(\bigcup_{j \in \mathcal{J}_i} C_{i, j}\Big), \quad i \in \mathcal{I},
\]
where \(\conv(\cdot)\) denotes the convex hull of a set. Assume that the sets \(C_i\) are disjoint, and that the distribution \(p_0\) is supported on a bounded, closed convex set \(S\). 

Then, there exists \(t_1 \in (0, 1)\) such that for any \(t \in (t_1, 1]\), the \(|\mathcal{I}|\) convex sets
\[
(1 - t) S + t C_i, \quad i \in \mathcal{I}
\]
are disjoint. Similarly, there exists \(t_2 \in (t_1, 1)\) such that for any \(t \in (t_2, 1]\), the \(\sum_{i \in \mathcal{I}} |\mathcal{J}_i|\) convex sets
\[
(1 - t) S + t C_{i, j}, \quad i \in \mathcal{I},\, j \in \mathcal{J}_i
\]
are disjoint.

Let \(\vphi_t\) denote the optimal flow under the OT paths. Then, for any \(\vx\) drawn from \(p_0\), if \(\vphi_{t_1}(\vx) \in (1 - t_1) S + t_1 C_i\), it follows that \(\vphi_t(\vx) \in (1 - t) S + t C_i\) for any \(t \in (t_1, 1]\). Specifically, \(\vphi_1(\vx) \in C_i\). Furthermore, if \(\vphi_{t_2}(\vx) \in (1 - t_2) S + t_2 C_{i,j}\), then \(\vphi_t(\vx) \in (1 - t) S + t C_{i,j}\) for any \(t \in (t_2, 1]\). Specifically, \(\vphi_1(\vx) \in C_{i,j}\).
\end{theorem}

\begin{proof}
We begin by recalling two important lemmas about convex sets, both of which are straightforward to prove.

\textbf{Lemma 1:} The scalar multiple of a convex set is also convex. For a scalar \(c\) and a convex set \(A\), the scalar multiple is defined as
\[
cA = \{c a \colon a \in A \}.
\]

\textbf{Lemma 2:} The Minkowski sum of two convex sets is convex. The Minkowski sum of sets \(A\) and \(B\) is given by
\[
A + B = \{a + b \colon a \in A,\, b \in B \}.
\]

Now, we proceed with the proof of the theorem. We will first establish the existence of \(t_1\); the existence of \(t_2\) follows analogously. For each pair \(i_1 \neq i_2\) in \(\mathcal{I}\), define the function
\[
d_{i_1, i_2}(t) = \dist\big( (1 - t) S + t C_{i_1}, (1 - t) S + t C_{i_2} \big),
\]
where \(\dist(A, B)\) denotes the Euclidean distance between the sets \(A\) and \(B\). Observe that \(d_{i_1, i_2}(t)\) is a continuous function of \(t\) on \([0, 1]\). This is due to the fact that the Minkowski sum and scalar multiplication are continuous operations, and the distance function is continuous with respect to the Hausdorff metric for closed sets in Euclidean space.

Next, define
\[
t_{i_1, i_2} = \max_{t \in [0, 1]} \big\{ t \colon d_{i_1, i_2}(t) = 0 \big\}.
\]
This value of \(t_{i_1, i_2}\) always exists because \(d_{i_1, i_2}(0) = \dist(S, S) = 0\), and since a continuous function defined on a compact set in Euclidean space must attain its maximum, \(t_{i_1, i_2}\) is well-defined.

We now show that \(t_{i_1, i_2} < 1\). In fact, a stronger result holds: the set
\[
T \coloneqq \{t \colon d_{i_1, i_2}(t) = 0\}
\]
is a closed interval, and since \(1\) does not belong to this set, we conclude that \(t_{i_1, i_2} < 1\).

To prove this stronger result, we first observe that \(T\) is closed because the distance function is continuous with respect to \(t\). Next, we prove that \(T\) is convex. For any \(u, v \in T\) and \(\lambda \in (0, 1)\), we need to show that \(\lambda u + (1 - \lambda) v \in T\). Consider the expression
\[
\big(1 - (\lambda u + (1 - \lambda) v)\big) S + \big(\lambda u + (1 - \lambda) v\big) C_{i_1}.
\]
Using the convexity properties of the Minkowski sum and scalar multiplication, we have that the above set equals
\[
\lambda \big((1 - u) S + u C_{i_1}\big) + (1 - \lambda) \big((1 - v) S + v C_{i_1}\big),
\]
and similarly for \(C_{i_2}\). Recall that the distance between a compact set and a disjoint closed set is strictly positive. Therefore, the two sets
\[
\lambda \big((1 - u) S + u C_{i_1}\big) + (1 - \lambda) \big((1 - v) S + v C_{i_1}\big)
\]
and
\[
\lambda \big((1 - u) S + u C_{i_2}\big) + (1 - \lambda) \big((1 - v) S + v C_{i_2}\big)
\]
intersect, which implies that \(\lambda u + (1 - \lambda) v \in T\).  Consequently, \(T\) is convex. Since \(T\) is a convex, bounded, and closed set on a line, it must be a closed interval. Thus, \(T\) is a closed interval, and \(t_{i_1, i_2} < 1\). Define
\[
t_1 = \max_{i_1, i_2 \in \mathcal{I}} t_{i_1, i_2}.
\]
Since there are finitely many pairs \((i_1, i_2)\), the maximum is well-defined, and \(t_1 < 1\). For any \(t > t_1\), the \(|\mathcal{I}|\) convex sets
\[
(1 - t) S + t C_i, \quad i \in \mathcal{I},
\]
are disjoint, as desired.

Since \(\vphi_t\) is the optimal flow under the OT paths, \cref{cor:marginal_probability_path} implies that
\[
\vphi_t(\vx) \in \bigcup_{i \in \mathcal{I}} \big( (1 - t) S + t C_i \big).
\]
Assume that at time \(t_1\), we have
\[
\vphi_{t_1}(\vx) \in (1 - t_1) S + t_1 C_i.
\]
Now, suppose for the sake of contradiction that there exists \(t' \in (t_1, 1]\) such that
\[
\vphi_{t'}(\vx) \notin (1 - t') S + t' C_i.
\]
Given that \(\vphi_{t'}(\vx) \in \bigcup_{j \in \mathcal{I}} \big( (1 - t') S + t' C_j \big)\), it must be that
\[
\vphi_{t'}(\vx) \in (1 - t') S + t' C_j
\]
for some \(j \neq i\). Now, define the distance function
\[
d(t) = \dist\big( \vphi_t(\vx), (1 - t) S + t C_i \big), \quad t \in [t_1, t'],
\]
which is continuous. We know that \(d(t_1) = 0\) and, since \(\vphi_{t'}(\vx) \notin (1 - t') S + t' C_i\), we have 
\[
d(t') \geq \dist\Big( (1 - t') S + t' C_i, (1 - t') S + t' C_j \Big) > 0.
\]
However, for any
\[
\vy \in \bigcup_{k \neq i \in \mathcal{I}} \big( (1 - t) S + t C_k \big),
\]
we have
\[
\dist\big( \vy, (1 - t) S + t C_i \big) \geq \min_{k \neq i} \dist\big( (1 - t) S + t C_k, (1 - t) S + t C_i \big) > 0, \quad t \in [t_1, t'].
\]
Since a continuous function defined on a compact set in Euclidean space must attain its minimum, the function \(\dist\big( \vy, (1 - t) S + t C_i \big)\) for \(t \in [t_1, t']\) has a positive lower bound. This implies that \(d(t)\) for \(t \in [t_1, t']\) must either remain at zero or jump to a positive value, but this contradicts the fact that \(d(t)\) is continuous and \(d(t') > 0\). Therefore, our assumption must be false, and we conclude that \(\vphi_t(\vx)\) remains in \((1 - t) S + t C_i\) for all \(t \in [t_1, 1]\). The same reasoning applies to \(C_{i, j}\), showing that if \(\vphi_{t_2}(\vx) \in (1 - t_2) S + t_2 C_{i,j}\), then \(\vphi_t(\vx)\) remains in \((1 - t) S + t C_{i, j}\) for all \(t \in [t_2, 1]\).
\end{proof}

\subsection{Memorization Phenomenon under Optimal Velocity Field}

\begin{theorem}[Memorization phenomenon under the optimal velocity field]
Consider the flow ODE
\[
\dfrac{\dif\vphi_t(\vx)}{\dif t} = \vv_t^*(\vphi_t(\vx)), \quad \vphi_0(\vx)=\vx, \quad t\in[0, 1],
\]
where
\[
\vv_t^*(\vx) = \sum_{i=1}^{N}\dfrac{\vy^i-\vx}{1-t}\cdot\dfrac{\exp\big(-\|(\vx-t\vy^i)/(1-t)\|_2^2/2\big)}{\sum_{j=1}^{N}\exp\big(-\|(\vx-t\vy^j)/(1-t)\|_2^2/2\big)}.
\]
Then the set \(\{\vphi_1(\vx)\colon \vx\in\mathbb{R}^d\}\) equals \(\{\vy^1, \vy^2, \dotsc, \vy^N\}\) up to a finite set.
\end{theorem}

\begin{proof}
The ODE can be rewritten as
\[
\dfrac{\dif\vphi_t(\vx)}{\dif t} = -\dfrac{\vphi_t(\vx)}{1-t}+\dfrac{1}{1-t}\cdot \mY\cdot\vw_t(\vphi_t(\vx)), \quad \vphi_0(\vx)=\vx, \quad t\in [0, 1],
\]
where \(\mY=[\vy^1, \vy^2, \dotsc, \vy^N]\), and
\[
\vw_t(\vx) = \softmax\Big(-\dfrac{1}{2}\Big\Vert\dfrac{\vx - t\vy^1}{1 - t}\Big\Vert_2^2, -\dfrac{1}{2}\Big\Vert\dfrac{\vx - t\vy^2}{1 - t}\Big\Vert_2^2, \dotsc, -\dfrac{1}{2}\Big\Vert\dfrac{\vx - t\vy^N}{1 - t}\Big\Vert_2^2\Big).
\]
This ODE can be solved by first isolating the homogeneous part, \(-\vphi_t(\vx)/(1-t)\), and then applying the method of variation of constants. The homogeneous equation
\[
\dfrac{\dif\vphi_t^{(h)}(\vx)}{\dif t} = -\dfrac{\vphi_t^{(h)}(\vx)}{1-t}, \quad \vphi_0^{(h)}(\vx)=\vx, \quad t\in[0, 1]
\]
has the general solution \(\vphi_t^{(h)}(\vx) = (1-t)\vx\). To account for the non-homogeneous term, we assume the solution takes the form \(\vphi_t(\vx) = (1-t) \vc_t(\vx)\), where \(\vc_t(\vx)\) is an unknown function. Substituting this into the original ODE and simplifying, we find that \(\vc_t(\vx)\) satisfies the following equation
\[
\frac{\dif \vc_t(\vx)}{\dif t} = \frac{\mY \cdot \vw_t(\vphi_t(\vx))}{(1-t)^2}, \quad \vc_0(\vx)=\vx, \quad t\in [0, 1].
\]
Integrating both sides from \(0\) to \(t\), we obtain
\[
\vc_t(\vx) = \vx + \int_0^t \frac{\mY \cdot \vw_s(\vphi_s(\vx))}{(1-s)^2} \dif s.
\]
Thus, the solution to the ODE is
\[
\vphi_t(\vx) = (1-t)\cdot\Big(\vx + \int_0^t \frac{\mY \cdot \vw_s(\vphi_s(\vx))}{(1-s)^2} \dif s\Big).
\]
We remark that this is not an explicit solution since \(\vphi_t(\vx)\) appears on both sides of the equation, particularly within the integral on the right-hand side. However, this implicit form is sufficient for addressing the problem at hand.

Now, we are interested in the behavior of \(\vphi_t(\vx)\) when \(t=1\). Specifically, we want to compute
\[
\vphi_1(\vx) = \lim_{t \to 1} (1-t) \cdot \Big(\int_0^t \frac{\mY \cdot \vw_s(\vphi_s(\vx))}{(1-s)^2} \dif s\Big).
\]
We observe that near \(t = 1\), the term \((1-t)\) approaches zero, while the integral tends to infinity due to singularity of \(1/(1-s)^2\). To evaluate this limit, we resort to L'H\^opital's rule. We differentiate both the numerator and denominator of the expression
\[
\frac{\int_0^t \mY \cdot \vw_s(\vphi_s(\vx))/(1-s)^2 \dif s}{1/(1-t)}
\]
with respect to \(t\). This yields
\[
\lim_{t \to 1} \frac{\mY \cdot \vw_t(\vphi_t(\vx))/(1-t)^2}{1/(1-t)^2} = \mY \cdot \vw_1(\vphi_1(\vx)).
\]
Thus, we conclude that
\[
\vphi_1(\vx) = \mY \cdot \vw_1(\vphi_1(\vx)).
\]
By substituting \( t = 1 \) into \( \vw_t(\vx) \), we obtain
\[
\vw_1(\vphi_1(\vx)) = \lim_{t \rightarrow 1} \softmax\Big(-\dfrac{1}{2}\Big\Vert\dfrac{\vphi_1(\vx) - \vy^1}{1 - t}\Big\Vert_2^2, -\dfrac{1}{2}\Big\Vert\dfrac{\vphi_1(\vx) - \vy^2}{1 - t}\Big\Vert_2^2, \dots, -\dfrac{1}{2}\Big\Vert\dfrac{\vphi_1(\vx) - \vy^N}{1 - t}\Big\Vert_2^2\Big).
\]
Due to the singularity of \(1/(1-t)\), \( \vw_1(\vx) \) becomes a one-hot vector where the nonzero entry corresponds to the index \( i \) that minimizes \( \|\vx - \vy^i\|_2 \). It is important to note that this conclusion excludes a set of measure zero in \( \mathbb{R}^d \) where two or more distances \( \|\vx - \vy^i\|_2 \) coincide. Consequently, the set \(\{\vphi_1(\vx)\colon \vx\in\mathbb{R}^d\}\) equals \(\{\vy^1, \vy^2, \dotsc, \vy^N\}\) up to a finite set, which completes our proof.
\end{proof}

\subsection{Optimal velocity field as an Special Instance of OSDNet}

\begin{proposition}[Optimal velocity field as a specific instance of OSDNet]
The optimal velocity field \(\vv_t^*(\vx)\) derived in \cref{thm:optimal_velocity_field_expression} is a specific instance of the OSDNet, where
\[
\hat{\mO}_t^* = -\dfrac{1}{1 - t} \cdot \mI_{d-D}
\]
and
\[
\hat{\vs}_t^*(\vx) = \dfrac{1}{1 - t} \cdot \big(\mR\cdot\vw_t(\vx) - \mV^\top \cdot \vx \big).
\]
\end{proposition}

\begin{proof}
We first express \(\vv_t^*(\vx)\) in matrix form and then align it with the structure of OSDNet:
\[
\begin{aligned}
\vv_t^*(\vx) &= \sum_{i=1}^{N}\dfrac{\vy^i-\vx}{1-t}\cdot\dfrac{\exp\big(-\|(\vx-t\vy^i)/(1-t)\|_2^2/2\big)}{\sum_{j=1}^{N}\exp\big(-\|(\vx-t\vy^j)/(1-t)\|_2^2/2\big)}\\
&= -\dfrac{\vx}{1-t} + \dfrac{1}{1-t}\cdot\mY\cdot\vw_t(\vx)\\
&= -\dfrac{1}{1-t}\cdot\big(\mV\mV^\top + \mV^\perp(\mV^\perp)^\top\big)\cdot\vx + \dfrac{1}{1-t}\cdot\mV\cdot\mR\cdot\vw_t(\vx)\\
&=\mV^\perp \cdot \Big(-\dfrac{1}{1 - t} \cdot \mI_{d-D}\Big) \cdot (\mV^\perp)^\top \cdot \vx + \mV\cdot\Big(\dfrac{1}{1 - t} \cdot \big(\mR\cdot\vw_t(\vx) - \mV^\top \cdot \vx \big)\Big)\\
&= \mV^\perp \cdot \hat{\mO}_t^* \cdot (\mV^\perp)^\top \cdot \vx + \mV \cdot \hat{\vs}_t^*(\vx),
\end{aligned}
\]
which completes the proof.
\end{proof}

\subsection{Teacher-Student Training of OSDNet}

\begin{proposition}[Teacher-student training of OSDNet]
Minimizing the CFM loss
\[
\mathcal{L}_{\text{CFM}}\big(\hat{\mO}_t, \hat{\vs}_t(\vx)\big) = \E_{t \sim \mathcal{U}[0,1], \vx_1 \sim p_1(\vx_1), \vx \sim p_t(\vx|\vx_1)} \big[ \big\lVert\hat{\vv}_t\big(\vx; \hat{\mO}_t, \hat{\vs}_t(\vx)\big) - \vu_t(\vx|\vx_1)\big\rVert_2^2 \big],
\]
is equivalent to minimizing the Teacher-Student Training (TST) loss
\[
\mathcal{L}_{\text{TST}}\big(\hat{\mO}_t, \hat{\vs}_t(\vx)\big) = \E_{t \sim \mathcal{U}[0,1], \vx \sim p_t(\vx)} \big[ \big\lVert\hat{\vv}_t\big(\vx; \hat{\mO}_t, \hat{\vs}_t(\vx)\big) - \vv_t^*(\vx)\big\rVert_2^2 \big],
\]
in the sense that \(\mathcal{L}_{\text{CFM}}\big(\hat{\mO}_t, \hat{\vs}_t(\vx)\big)\) and \(\mathcal{L}_{\text{TST}}\big(\hat{\mO}_t, \hat{\vs}_t(\vx)\big)\) differ only by a constant. As a result, training OSDNet reduces to the two independent minimization problems below:
\[
\min_{\hat{\mO}_t}\E_{t\sim\mathcal{U}[0,1], \vx\sim p_t(\vx)}\Big[\Big\lVert \mV^\perp \cdot\Big(\hat{\mO}_t +\dfrac{1}{1-t}\cdot \mI_{d-D}\Big)\cdot (\mV^\perp)^\top \cdot \vx\Big\rVert_2^2\Big]
\]
and
\[
\min_{\hat{\vs}_t(\vx)}\E_{t\sim\mathcal{U}[0,1], \vx\sim p_t(\vx)}\Big[\Big\lVert\mV\cdot\Big(\hat{\vs}_t(\vx)-\dfrac{1}{1 - t} \cdot \big( \mR \cdot \vw_t(\vx) - \mV^\top \cdot \vx \big)\Big)\Big\rVert_2^2\Big].
\]
\end{proposition}

\begin{proof}
Recall that \(\vv_t^*(\vx) = \vu_t(\vx)\). The first part of the proof has already been established in \citep{lipman2023flow}. For the second part, we substitute the expression for \(\vv_t^*(\vx)\) into the expectation term in \(\mathcal{L}_{\text{TST}}\), and derive that it simplifies to
\[
\begin{aligned}
&\Big\lVert \mV^\perp \cdot\Big(\hat{\mO}_t + \dfrac{1}{1 - t} \cdot \mI_{d-D}\Big)\cdot (\mV^\perp)^\top\cdot \vx + \mV \cdot\Big( \hat{s}_t(\vx) - \dfrac{1}{1 - t} \cdot \big( \mR \cdot \vw_t(\vx) - \mY^\top \cdot \vx \big)\Big)\Big\rVert_2^2 \\
=\,& \Big\lVert \mV^\perp \cdot\Big(\hat{\mO}_t + \dfrac{1}{1 - t} \cdot \mI_{d-D}\Big)\cdot (\mV^\perp)^\top \cdot \vx \Big\rVert_2^2 + \Big\lVert \mV \cdot \Big( \hat{s}_t(\vx) - \dfrac{1}{1 - t} \cdot \big( \mR\cdot \vw_t(\vx) - \mY^\top \cdot \vx \big)\Big)\Big\rVert_2^2,
\end{aligned}
\]
which follows from the application of the Pythagorean theorem.
\end{proof}

\subsection{Decay of Off-Subspace Component}

\begin{proposition}[Teacher-student training of \(\hat{\mO}_t\)]
The solution to the minimization problem
\[
\min_{\hat{\mO}_t} \E_{t \sim \mathcal{U}[0,1],\ \vx \sim p_t(\vx)} \Big[ \Big\lVert \mV^\perp\cdot\Big( \hat{\mO}_t + \dfrac{1}{1 - t} \mI_{d-D} \Big) \cdot (\mV^\perp)^\top \cdot\vx \Big\rVert_2^2 \Big]
\]
reduces to solving the function approximation problem
\[
\min_{\hat{\mO}_t} \int_0^1 \big\lVert (1 - t)\cdot\hat{\mO}_t + \mI_{d-D} \big\rVert_{\mathrm{F}}^2 \dif t.
\]
\end{proposition}

\begin{proof}
We begin by deriving the distribution of \((\mV^\perp)^\top \vx\). Recall that \(\vx = (1 - t) \vx_0 + t \vx_1\), where \(\vx_0 \sim \mathcal{N}(\vzero, \mI_d)\) is a standard Gaussian random vector, and \(\vx_1\) lies within the span of \(\mV\). Since \(\mV^\perp\) projects onto the orthogonal complement of \(\range(\mV)\), it follows that \((\mV^\perp)^\top \vx_1 = \vzero\). Thus, we have
\[
(\mV^\perp)^\top \vx = (\mV^\perp)^\top ( (1 - t) \vx_0 + t \vx_1) = (1 - t) (\mV^\perp)^\top \vx_0 + t (\mV^\perp)^\top \vx_1 = (1 - t) (\mV^\perp)^\top \vx_0.
\]
Since \(\vx_0\) is standard Gaussian, it follows that
\[
(\mV^\perp)^\top \vx \sim \mathcal{N}\big( \vzero, (1 - t)^2 (\mV^\perp)^\top\mV^\perp \big) = \mathcal{N}(\vzero, (1-t)^2\mI_{d-D}).
\]

Next, we consider the expected squared norm of the objective function:
\[
\E_{t \sim \mathcal{U}[0,1],\vx \sim p_t(\vx)} \Big[ \Big\lVert \mV^\perp\cdot\Big( \hat{\mO}_t + \dfrac{1}{1 - t} \mI_{d-D} \Big) \cdot (\mV^\perp)^\top \cdot\vx \Big\rVert_2^2 \Big] = \E_{t \sim \mathcal{U}[0,1],\ \vx \sim p_t(\vx)} \Big[ \Big\lVert \Big( \hat{\mO}_t + \dfrac{1}{1 - t} \mI_{d-D} \Big) \cdot (\mV^\perp)^\top \cdot\vx \Big\rVert_2^2 \Big].
\]
Using properties of Gaussian random variables, this expectation simplifies to
\[
\E_{t \sim \mathcal{U}[0,1],\vx \sim p_t(\vx)} \Big[ \Big\lVert \Big( \hat{\mO}_t + \dfrac{1}{1 - t} \mI_{d-D} \Big) (\mV^\perp)^\top\vx \Big\rVert_2^2 \Big] = (1 - t)^2 \cdot \Big\lVert \hat{\mO}_t + \dfrac{1}{1 - t} \mI_{d-D} \Big\rVert_{\mathrm{F}}^2.
\]

Thus, the original minimization problem reduces to
\[
\min_{\hat{\mO}_t} \int_0^1 \big\lVert (1 - t)\cdot\hat{\mO}_t + \mI_{d-D} \big\rVert_{\mathrm{F}}^2 \dif t.
\]
This completes the proof.
\end{proof}

\subsection{Generalization of Subspace Components}

\begin{lemma}[Boundedness of the Jacobian and Hessian under smooth vector fields]
\label{lemma:boundedness_of_jacobian_and_hessian}
Let \( \vz_t^\vx \in \mathbb{R}^d \) be the solution of the ODE
\[
\dfrac{\dif \vz_t^\vx}{\dif t} = \vv(\vz_t^\vx, t), \quad \vz_0^\vx = \vx,\quad t\in [0, T],
\]
where \( \vv\colon \mathbb{R}^d \times [0, T] \to \mathbb{R}^n \) is a smooth vector field. In the expression \(\vz_t^\vx\), the superscript \(\vx\) represents the initial value, and the subscript \(t\) denotes the time. Assume that \(\|\nabla_\vz \vv(\vz_t^\vx, t)\|_{2} \leq L\) and \(\|\nabla_\vz^2 \vv(\vz_t^\vx, t)\|_{\mathrm{F}} \leq M\) for some constants \(L, M > 0\). Then, the Jacobian \(\nabla_\vx \vz_t^\vx\) satisfies the ODE
\[
\dfrac{\dif}{\dif t} \nabla_\vx \vz_t^\vx = \nabla_\vz \vv(\vz_t^\vx, t) \cdot \nabla_\vx \vz_t^\vx, \quad \nabla_\vx \vz_0^\vx = \mI_d, \quad t\in [0, T].
\]
The Hessian \(\nabla_\vx^2 \vz_t^\vx\) satisfies the ODE
\[
\dfrac{\dif}{\dif t} \nabla_\vx^2 \vz_t^\vx = \nabla_\vz^2 \vv(\vz_t^\vx, t) [ \nabla_\vx \vz_t^\vx, \nabla_\vx \vz_t^\vx ] + \nabla_\vz \vv(\vz_t^\vx, t) \cdot \nabla_\vx^2 \vz_t^\vx, \quad \nabla_\vx^2 \vz_0^\vx = \vzero, \quad t\in [0, T].
\]
Furthermore, the norms of \(\nabla_\vx \vz_t^\vx\) and \(\nabla_\vx^2 \vz_t^\vx\) are bounded by
\[
\|\nabla_\vx \vz_t^\vx\|_{2} \leq \exp(Lt), \quad \|\nabla_\vx^2 \vz_t^\vx\|_{\mathrm{F}} \leq \frac{M}{L}\cdot\big(\exp(2Lt)-\exp(Lt)\big),
\]
respectively.
\end{lemma}

\begin{proof}
First, we differentiate both sides of the original ODE with respect to the initial condition \(\vx\). Applying the chain rule, we get
\[
\dfrac{\dif}{\dif t}\nabla_\vx \vz_t^\vx = \nabla_\vz \vv(\vz_t^\vx, t) \cdot \nabla_\vx \vz_t^\vx,
\]
where \(\nabla_\vz \vv(\vz_t^\vx, t) \in \mathbb{R}^{d \times d}\) is the Jacobian matrix of \(\vv(\vz_t^\vx, t)\) with respect to the first entry \(\vz\), and the \(\cdot\) denotes regular matrix multiplication. The initial condition for this Jacobian is
\[
\nabla_\vx \vz_0^\vx = \mI_d.
\]
Thus, the Jacobian \(\nabla_\vx \vz_t^\vx\) satisfies the ODE
\[
\dfrac{\dif}{\dif t} \nabla_\vx \vz_t^\vx = \nabla_\vz \vv(\vz_t^\vx, t) \cdot \nabla_\vx \vz_t^\vx, \quad \nabla_x \vz_0^\vx = \mI_d, \quad t\in [0, T].
\]

Next, we differentiate the ODE for the Jacobian \(\nabla_\vx \vz_t^\vx\) with respect to \(\vx\) to obtain the ODE for the Hessian \(\nabla_\vx^2 \vz_t^\vx\):
\[
\dfrac{\dif}{\dif t} \nabla_\vx^2 \vz_t^\vx = \nabla_\vx ( \nabla_\vz \vv(\vz_t^\vx, t) \cdot \nabla_\vx \vz_t^\vx ).
\]
Using the product rule, we can expand the expression in the right-hand side as
\[
\nabla_\vx \nabla_\vz \vv(\vz_t^\vx, t) \cdot \nabla_\vx \vz_t^\vx + \nabla_\vz \vv(\vz_t^\vx, t) \cdot \nabla_\vx^2 \vz_t^\vx.
\]
The first term is computed using the chain rule, and can be written in index notation as
\[
\big(\nabla_\vx \nabla_\vz \vv(\vz_t^\vx, t) \cdot \nabla_\vx \vz_t^\vx\big)_{i,j,k} 
= \sum_{l,s=1}^{d}\big(\nabla_\vz^2 \vv(\vz_t^\vx, t)\big)_{i,l,s}\cdot\big(\nabla_\vx \vz_t^\vx\big)_{l,j}\cdot\big(\nabla_\vx \vz_t^\vx\big)_{s,k}.
\]
For brevity of notation, we denote it as
\[
\nabla_\vx \nabla_\vz \vv(\vz_t^\vx, t) \cdot \nabla_\vx \vz_t^\vx = \nabla_\vz^2 \vv(\vz_t^\vx, t) [ \nabla_\vx \vz_t^\vx, \nabla_\vx \vz_t^\vx ].
\]
The second term can be written in index notation as
\[
\big(\nabla_\vz \vv(\vz_t^\vx, t) \cdot \nabla_\vx^2 \vz_t^\vx\big)_{i,j,k} = \sum_{l=1}^{d}\big(\nabla_\vx \vz_t^\vx\big)_{i,l}\cdot\big(\nabla_\vx^2 \vz_t^\vx\big)_{l,j,k}.
\]
Together, the Hessian \(\nabla_\vx^2 \vz_t^\vx\) satisfies the ODE
\[
\dfrac{\dif}{\dif t} \nabla_\vx^2 \vz_t^\vx = \nabla_\vz^2 \vv(\vz_t^\vx, t) [ \nabla_\vx \vz_t^\vx, \nabla_\vx \vz_t^\vx ] + \nabla_\vz \vv(\vz_t^\vx, t) \cdot \nabla_\vx^2 \vz_t^\vx, \quad \nabla_\vx^2 \vz_0^\vx = \vzero, \quad t\in [0, T].
\]

We now derive bounds for the operator norm of the Jacobian \(\nabla_\vx \vz_t^\vx\) and Hessian \(\nabla_\vx^2 \vz_t^\vx\). First, we consider the Jacobian ODE
\[
\dfrac{\dif}{\dif t} \|\nabla_\vx \vz_t^\vx\|_{2} \leq \|\nabla_\vz \vv(\vz_t^\vx, t)\|_{2} \cdot \|\nabla_\vx \vz_t^\vx\|_{2}\leq L\cdot \|\nabla_\vx \vz_t^\vx\|_{2}.
\]
Applying Gr\"onwall's inequality, we get
\[
\|\nabla_\vx \vz_t^\vx\|_{2} \leq \exp(Lt).
\]
Next, for the Hessian, the Hessian ODE gives
\[
\dfrac{\dif}{\dif t} \|\nabla_\vx^2 \vz_t^\vx\|_{\mathrm{F}} \leq \|\nabla_\vz^2 \vv(\vz_t^\vx, t) [ \nabla_\vx \vz_t^\vx, \nabla_\vx \vz_t^\vx ]\|_{\mathrm{F}} + \|\nabla_\vz \vv(\vz_t^\vx, t)\cdot\nabla_\vx^2 \vz_t^\vx\|_{\mathrm{F}}.
\]
For the first term, since \(\|\nabla_\vz^2 \vv(\vz_t^\vx, t)\|_{\mathrm{F}} \leq M\) and \(\|\nabla_\vx \vz_t^\vx\|_{2} \leq \exp(Lt)\), we have
\[
\|\nabla_\vz^2 \vv(\vz_t^\vx, t) [ \nabla_\vx \vz_t^\vx, \nabla_\vx \vz_t^\vx ]\|_{\mathrm{F}} \leq M\cdot\exp(2Lt).
\]
For the second term, we have
\[
\|\nabla_\vz \vv(\vz_t^\vx, t)\cdot\nabla_\vx^2 \vz_t^\vx\|_{\mathrm{F}}\leq L\cdot \|\nabla_\vx^2 \vz_t^\vx\|_{\mathrm{F}}.
\]
Thus, we have the inequality:
\[
\dfrac{\dif}{\dif t} \|\nabla_\vx^2 \vz_t^\vx\|_{\mathrm{F}} \leq M\cdot\exp(2Lt) + L\cdot\|\nabla_\vx^2 \vz_t^\vx\|_{\mathrm{F}}.
\]
To solve this, we use the integrating factor \(\exp(-Lt)\), giving
\[
\dfrac{\dif}{\dif t} \big(\exp(-Lt)\cdot\|\nabla_\vx^2 \vz_t^\vx\|_{\mathrm{F}} \big) \leq M\cdot\exp(Lt).
\]
Integrating both sides, we obtain
\[
\|\nabla_\vx^2 \vz_t^\vx\|_{\mathrm{F}} \leq \frac{M}{L}\cdot\big(\exp(2Lt)-\exp(Lt)\big).
\]
This completes the proof.
\end{proof}

\begin{theorem}[Generalization of subspace components]
Let \(\vphi_t^*(\vx)\) and \(\hat{\vphi}_t(\vx)\) be the solutions to the following ODE and SDE, respectively, defined on the interval \(t \in [0, 1 - \varepsilon]\), where \(0 < \varepsilon \leq 1\):
\[
\begin{aligned}
\dif \vphi_t^*(\vx) &= \mV \cdot \vs_t^*(\vphi_t^*(\vx))\dif t, \quad \vphi_0^*(\vx) = \vx, \\
\dif \hat{\vphi}_t(\vx) &= \mV \cdot \hat{\vs}_t(\hat{\vphi}_t(\vx))\dif t + \sigma\dif \vW_t, \quad \hat{\vphi}_0(\vx) = \vx,
\end{aligned}
\]
where \(\vW_t\) is a standard Wiener process in \(\mathbb{R}^d\), \(\mV\) is a matrix with orthonormal columns. Assume that the norms of the Jacobians and Hessians of \(\vs_t^*(\vx)\) are bounded:
\[
\|\nabla_\vx \vs_t^*(\vx)\|_{2} \leq L, \quad \|\nabla_\vx^2 \vs_t^*(\vx)\|_{\mathrm{F}} \leq M,
\]
for some constants \(L, M > 0\). Then, the difference between \(\vphi_{1 - \varepsilon}^*(\vx)\) and \(\hat{\vphi}_{1 - \varepsilon}(\vx)\) is given by
\[
\begin{aligned}
\vphi_{1 - \varepsilon}^*(\vx) - \hat{\vphi}_{1 - \varepsilon}(\vx) 
=& \int_0^{1 - \varepsilon} \nabla_\vx \vphi_{t, 1 - \varepsilon}^*(\hat{\vphi}_t(\vx)) \cdot \mV \cdot\big( \vs_t^*(\hat{\vphi}_t(\vx)) - \hat{\vs}_t(\hat{\vphi}_t(\vx)) \big) \dif t \\
& - \sigma\cdot\int_0^{1 - \varepsilon} \nabla_\vx \vphi_{t, 1 - \varepsilon}^*(\hat{\vphi}_t(\vx))\dif \vW_t 
- \dfrac{\sigma^2}{2}\cdot\int_0^{1 - \varepsilon}\sum_{i,j=1}^{d} \big(\nabla_\vx^2 \vphi_{t, 1 - \varepsilon}^*(\hat{\vphi}_t(\vx))\big)_{i,j,\colon} \dif t,
\end{aligned}
\]
where \(\vphi_{s, t}^*(\vx)\) denotes the solution to the ODE from time \(s\) to \(t\) with initial condition \(\vx\). Moreover, the expected squared difference satisfies the inequality
\[
\E[ \| \vphi_{1 - \varepsilon}^*(\vx) - \hat{\vphi}_{1 - \varepsilon}(\vx) \|_2^2 ] \leq C_1\cdot \int_0^{1 - \varepsilon} \E[ \| \vs_t^*(\hat{\vphi}_t(\vx)) - \hat{\vs}_t(\hat{\vphi}_t(\vx)) \|_2^2 ] \dif t + C_2\cdot\sigma^2 + C_3\cdot\sigma^4,
\]
for some constant \(C_1\), \(C_2\), and \(C_3\) depending on \(L\), \(M\), \(d\), and \(\varepsilon\). Here, the expectation is taken over the realizations of the standard Wiener process \(\vW_t\). 

Specifically, for \(\sigma = 0\), the inequality simplifies to the deterministic case
\[
\| \vphi_{1 - \varepsilon}^*(\vx) - \hat{\vphi}_{1 - \varepsilon}(\vx) \|_2^2 \leq C_4 \cdot \int_0^{1 - \varepsilon} \| \vs_t^*(\hat{\vphi}_t(\vx)) - \hat{\vs}_t(\hat{\vphi}_t(\vx)) \|_2^2  \dif t,
\]
where \(C_4\) is a constant depending only on \(L\) and \(\varepsilon\). 

Taking the expectation with respect to \(\vx \sim \mathcal{N}(\vzero, \mI_d)\), we obtain
\[
\E_{\vx \sim \mathcal{N}(\vzero, \mI_d)}[ \| \vphi_{1 - \varepsilon}^*(\vx) - \hat{\vphi}_{1 - \varepsilon}(\vx) \|_2^2 ] \leq C_5 \cdot \E_{t \sim \mathcal{U}(0, 1), \vx \sim p_t}[ \| \vs_t^*(\vx) - \hat{\vs}_t(\vx) \|_2^2 ],
\]
where \(p_t\) is the distribution of \(\vphi_t^*(\vx)\) and \(C_5\) is a constant depending only on the Lipschitz constant of \(\hat{\vs}_t(\vx)\) and \(\varepsilon\).
\end{theorem}

\begin{proof}
The difference between \(\vphi_{1 - \varepsilon}^*(\vx)\) and \(\hat{\vphi}_{1 - \varepsilon}(\vx)\) can be derived using the It\^o--Alekseev--Gr\"obner formula \citep{hudde2024ito}, which relates solutions of perturbed differential equations. In this context, the Skorohod integral reduces to the Itô integral because the integrand \(\nabla_\vx \vphi_{t, 1 - \varepsilon}^*(\hat{\vphi}_t(\vx))\) is adapted to the filtration generated by \(\hat{\vphi}_t(\vx)\).

Denote the three terms in \(\vphi_{1 - \varepsilon}^*(\vx) - \hat{\vphi}_{1 - \varepsilon}(\vx)\) by \(I_1\), \(I_2\), and \(I_3\), respectively:
\[
\begin{aligned}
I_1 &= \int_0^{1 - \varepsilon} \nabla_\vx \vphi_{t, 1 - \varepsilon}^*(\hat{\vphi}_t(\vx)) \cdot \mV \cdot \big( \vs_t^*(\hat{\vphi}_t(\vx)) - \hat{\vs}_t(\hat{\vphi}_t(\vx)) \big) \dif t,\\
I_2 &= - \sigma\cdot\int_0^{1 - \varepsilon} \nabla_\vx \vphi_{t, 1 - \varepsilon}^*(\hat{\vphi}_t(\vx))\dif \vW_t,\\
I_3 &= - \dfrac{\sigma^2}{2}\cdot\int_0^{1 - \varepsilon}\sum_{i,j=1}^{d} \big(\nabla_\vx^2 \vphi_{t, 1 - \varepsilon}^*(\hat{\vphi}_t(\vx))\big)_{i,j,\colon} \dif t.
\end{aligned}
\]

Our goal is to bound the expected squared norm of the difference,
\[
\E[ \| \vphi_{1 - \varepsilon}^*(\vx) - \hat{\vphi}_{1 - \varepsilon}(\vx) \|_2^2 ] = \E[ \| I_1 + I_2 + I_3 \|_2^2 ],
\]
where the expectation is taken over the realizations of the standard Wiener process \(\vW_t\).
Using the inequality 
\[
\|a + b + c\|_2^2 \leq 3\big(\|a\|_2^2 + \|b\|_2^2 + \|c\|_2^2 \big),
\]
we have
\[
\E[ \| I_1 + I_2 + I_3 \|_2^2 ] \leq 3\big(\E[ \|I_1\|_2^2 ] + \E[ \|I_2\|_2^2 ] + \E[ \|I_3\|_2^2 ] \big).
\]
We will bound each term separately. 

For \(I_1\), since \(\|\nabla_\vx \vs_t^*(\vx)\|_{2} \leq L\), the Jacobian \(\nabla_\vx \vphi_{t, s}^*(\vx)\) satisfies the bound (see \cref{lemma:boundedness_of_jacobian_and_hessian}):
\[
\|\nabla_\vx \vphi_{t, s}^*(\vx)\|_{2} \leq \exp( L(s-t) ) \leq \exp( L ).
\]
Moreover, because \(\mV\) has orthonormal columns, \(\|\mV\|_{2} = 1\). Therefore,
\[
\|I_1\|_2 \leq \int_0^{1 - \varepsilon} \| \nabla_\vx \vphi_{t, 1 - \varepsilon}^*(\hat{\vphi}_t(\vx)) \|_{2} \cdot \| \vs_t^*(\hat{\vphi}_t(\vx)) - \hat{\vs}_t(\hat{\vphi}_t(\vx)) \|_2 \dif t \leq \exp(L)\cdot\int_0^{1 - \varepsilon} \| \vs_t^*(\hat{\vphi}_t(\vx)) - \hat{\vs}_t(\hat{\vphi}_t(\vx)) \|_2 \dif t.
\]
Applying the Cauchy--Schwarz inequality,
\[
\|I_1\|_2^2 \leq \exp(2L)\cdot(1 - \varepsilon)\cdot\int_0^{1 - \varepsilon} \| \vs_t^*(\hat{\vphi}_t(\vx)) - \hat{\vs}_t(\hat{\vphi}_t(\vx)) \|_2^2 \dif t.
\]
Taking expectations,
\[
\E[ \|I_1\|_2^2 ] \leq \exp(2L)\cdot(1 - \varepsilon)\cdot\int_0^{1 - \varepsilon} \E[ \| \vs_t^*(\hat{\vphi}_t(\vx)) - \hat{\vs}_t(\hat{\vphi}_t(\vx)) \|_2^2 ] \dif t.
\]

For \(I_2\), using the It\^o isometry,
\[
\E[ \|I_2\|_2^2 ] = \sigma^2\cdot \E\Big[ \int_0^{1 - \varepsilon} \| \nabla_\vx \vphi_{t, 1 - \varepsilon}^*(\hat{\vphi}_t(\vx)) \|_{\mathrm{F}}^2 \dif t \Big].
\]
Since the \(2\)-norm bounds the Frobenius norm,
\[
\| \nabla_\vx \vphi_{t, 1 - \varepsilon}^*(\hat{\vphi}_t(\vx)) \|_{\mathrm{F}}^2 \leq d\cdot\| \nabla_\vx \vphi_{t, 1 - \varepsilon}^*(\hat{\vphi}_t(\vx)) \|_{2}^2 \leq d\cdot\exp(2L).
\]
Therefore,
\[
\E[ \|I_2\|_2^2 ] \leq \sigma^2\cdot d\cdot\exp(2L)\cdot(1 - \varepsilon).
\]

For \(I_3\), by the bound on the Hessian \(\nabla_\vx^2 \vphi_{t, s}^*(\vx)\) (see \cref{lemma:boundedness_of_jacobian_and_hessian}),
\[
\| \nabla_\vx^2 \vphi_{t, s}^*(\vx) \|_{\mathrm{F}} \leq \dfrac{M}{L}\cdot\big(\exp(2L(s-t))-\exp(L(s-t))\big)\leq \dfrac{M}{L}\cdot\big(\exp(2L)-\exp(L)\big).
\]
Since
\[
\Big\|\sum_{i,j=1}^d\big(\nabla_\vx^2 \vphi_{t, 1 - \varepsilon}^*(\hat{\vphi}_t(\vx))\big)_{i,j,\colon}\Big\|_{2}\leq \sqrt{d}\cdot\|\nabla_\vx^2 \vphi_{t, s}^*(\vx) \|_{\mathrm{F}},
\]
we have
\[
\|I_3\|_2 \leq \dfrac{\sigma^2}{2} \int_0^{1 - \varepsilon} \sqrt{d}\cdot\| \nabla_\vx^2 \vphi_{t, 1 - \varepsilon}^*(\hat{\vphi}_t(\vx)) \|_{\mathrm{F}} \dif t \leq \dfrac{\sigma^2\sqrt{d}M}{2L}\cdot\big(\exp(2L)-\exp(L)\big).
\]
Thus,
\[
\E[ \|I_3\|_2^2 ] \leq \dfrac{\sigma^4dM^2}{4L^2}\cdot\big(\exp(2L)-\exp(L)\big)^2.
\]

Adding up the bounds, we obtain
\[
\E[ \| \vphi_{1 - \varepsilon}^*(\vx) - \hat{\vphi}_{1 - \varepsilon}(\vx) \|_2^2 ] \leq C_1\cdot \int_0^{1 - \varepsilon} \E[ \| \vs_t^*(\hat{\vphi}_t(\vx)) - \hat{\vs}_t(\hat{\vphi}_t(\vx)) \|_2^2 ] \dif t + C_2\cdot\sigma^2 + C_3\cdot\sigma^4,
\]
where \(C_1\), \(C_2\), and \(C_3\) are constants that depend on \(L\), \(M\), \(d\), and \(\varepsilon\).

Specifically, for \(\sigma = 0\), the inequality simplifies to the deterministic case
\[
\| \vphi_{1 - \varepsilon}^*(\vx) - \hat{\vphi}_{1 - \varepsilon}(\vx) \|_2^2 \leq C_4 \cdot \int_0^{1 - \varepsilon} \| \vs_t^*(\hat{\vphi}_t(\vx)) - \hat{\vs}_t(\hat{\vphi}_t(\vx)) \|_2^2  \dif t,
\]
where \(C_4 = \exp(2L)\cdot(1-\varepsilon)\) is a constant depending only on \(L\) and \(\varepsilon\). 

Taking the expectation with respect to \(\vx \sim \mathcal{N}(\vzero, \mI_d)\), and noting that \(\hat{\vphi}_t(\vx)\) and \(\vphi_t^*(\vx)\) evolve according to their respective dynamics, we can interchange the roles of \(\hat{\vphi}_t(\vx)\) and \(\vphi_t^*(\vx)\) in the integrals. This leads to
\[
\E_{\vx \sim \mathcal{N}(\vzero, \mI_d)}[ \| \vphi_{1 - \varepsilon}^*(\vx) - \hat{\vphi}_{1 - \varepsilon}(\vx) \|_2^2 ] \leq C_5 \cdot \E_{t \sim \mathcal{U}(0, 1), \vx \sim p_t}[ \| \vs_t^*(\vx) - \hat{\vs}_t(\vx) \|_2^2 ],
\]
where \(p_t\) is the distribution of \(\vphi_t^*(\vx)\) and \(C_5\) is a constant depending only on the Lipschitz constant of \(\hat{\vs}_t(\vx)\) and \(\varepsilon\).
\end{proof}

\section{Numerical Experiments}
\label{app:numerical_experiments}

In this section, we begin with a recapitulation of the OSDNet network class, categorizing the two parameters, \(\hat{\mO}_t\) and \(\hat{\vs}_t(\vx)\), into their optimal and learned versions. We then present the details and results of numerical experiments, which include: (\romannumeral 1) the path geometry under the optimal velocity field (with both \(\hat{\mO}_t\) and \(\hat{\vs}_t(\vx)\) set to their optimal values); (\romannumeral 2) the decay of off-subspace components (with \(\hat{\vs}_t(\vx)\) optimal and \(\hat{\mO}_t\) learned); and (\romannumeral 3) the generalization of subspace components (with \(\hat{\mO}_t\) optimal and \(\hat{\vs}_t(\vx)\) learned).

\subsection{Architecture of the OSDNet}
Let \( \{\vy^i \colon 1 \leq i \leq N\} \subset \mathbb{R}^d \) represent real data points, which form the columns of the matrix \(\mY = [ \vy^1, \vy^2, \dotsc, \vy^N ]\). We obtain the matrix \(\mV\in\mathbb{R}^{d\times D}\), with orthonormal columns spanning the same subspace as these data points, by performing a reduced singular value decomposition on \(\mY\), i.e., \(\mY = \mV\mR\), where \(\mR\in\mathbb{R}^{D\times N}\) is the product of a diagonal matrix (containing the non-zero singular values of \(\mY\)) and a matrix with orthonormal rows (containing the right singular vectors). The OSDNet is given by
\[
\hat{\vv}_t\big(\vx; \hat{\mO}_t, \hat{\vs}_t(\vx)\big) = \mV^\perp \cdot \hat{\mO}_t \cdot (\mV^\perp)^\top \cdot \vx + \mV \cdot \hat{\vs}_t(\vx),
\]
where \(\hat{\mO}_t\) is a time-dependent diagonal matrix, and \(\hat{\vs}_t(\vx)\) is a time-dependent vector function, both of which are trainable parameters. As shown in \cref{prop:optimal_velocity_field_as_a_specific_instance_of_OSDNet}, the optimal velocity field is a specific case of OSDNet, where
\[
\hat{\mO}_t^* = -\dfrac{1}{1 - t} \cdot \mI_{d-D}
\]
and
\[
\hat{\vs}_t^*(\vx) = \dfrac{1}{1 - t} \cdot \big(\mR \cdot \vw_t(\vx) - \mV^\top \cdot \vx \big).
\]
The parameters \(\hat{\mO}_t\) and \(\hat{\vs}_t(\vx)\) in the OSDNet can be categorized into two types: the optimal version and the learned version. The learned version can employ any neural network architecture, as long as the input-output dimensionality constraints are satisfied. The optimal version, on the other hand, is derived directly from theoretical formulations and serves to verify the properties of the optimal velocity field. 
Having these optimal formulations provides a valuable reference point. In particular, by fixing one component to its optimal form, we can observe how the learned component influences the generated samples in the subspaces spanned by \(\mV\) and \(\mV^\perp\).

\subsection{Path Geometry under the Optimal Velocity Field}

In this subsection, we provide visualizations of generation paths under the optimal velocity field, which corresponds to \cref{subsec:path_geometry_under_the_optimal_velocity_field}.

\paragraph{Sparse and well separated datasets.}

In this setting, we set \( d = D = 2 \) to enable visualization. We randomly sample \( N = 6 \) points uniformly from the region \([-10, 10] \times [-10, 10]\). Both \( \hat{\mO}_t \) and \( \hat{\vs}_t(\vx) \) are configured to their optimal versions to implement the optimal velocity field. We plot in \cref{fig:exp_path_geometry_straight_trajectory} the generation paths, which are derived by discretizing the flow ODE into \(100\) steps, starting from various noise samples (green crosses). These paths, represented by colored trajectories, evolve over time \( t \), with brighter colors indicating larger values of \( t \). As \( t \) progresses, each path straightens and converges towards specific real data points (black triangles), effectively verifying \cref{thm:path_geometry_straight}.

\begin{figure}[htb]
\centering
\includegraphics[width=0.4\linewidth]{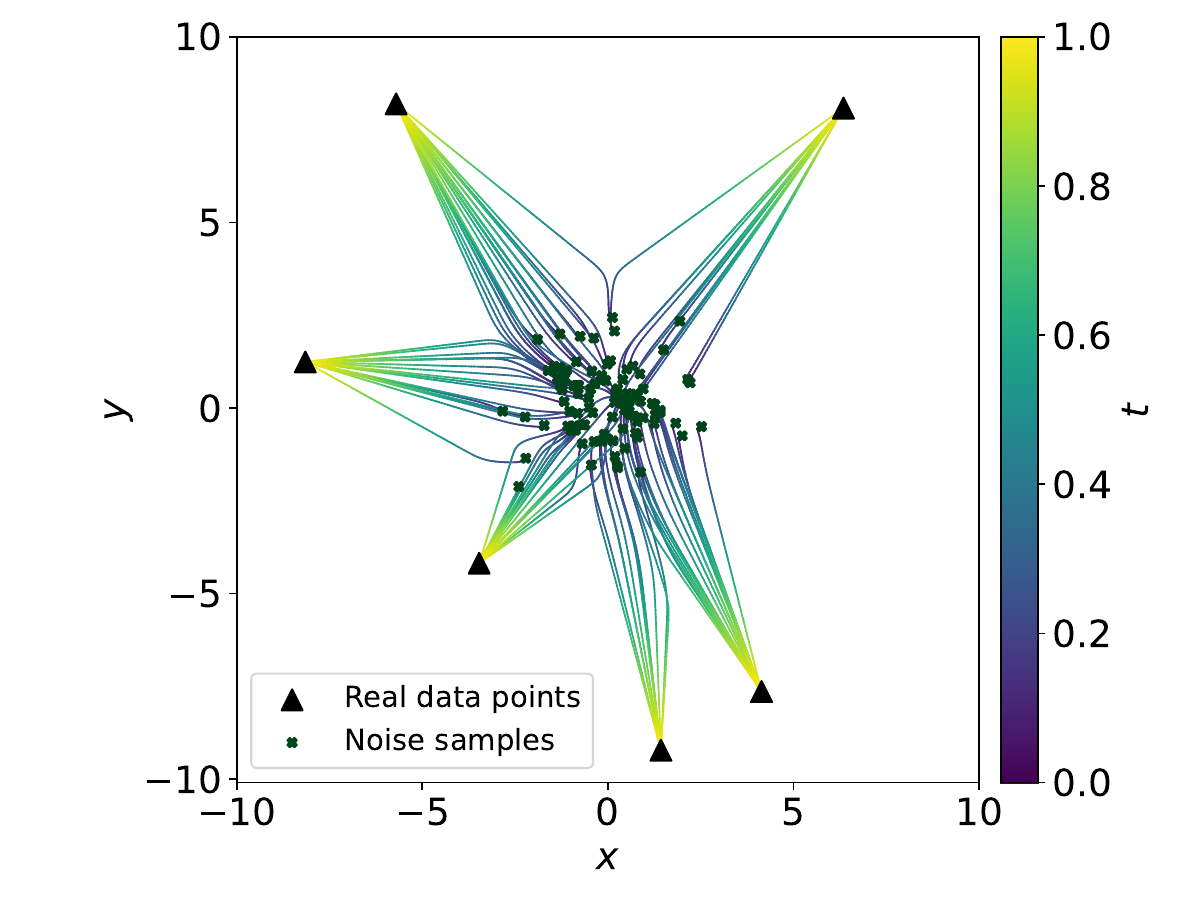}
\caption{The generation paths of Flow Matching models applied to sparse and well separated datasets under the optimal velocity field. The black triangles represent the real data points, which are \( N = 6 \) points randomly sampled from the region \([-10, 10] \times [-10, 10]\). The \textcolor[rgb]{0.0, 0.27, 0.11}{green crosses} denote noise samples. The colored trajectories depict generation paths starting from various noise samples, with brighter colors indicating larger \(t\). As \(t\) increases, these paths straighten, converging towards specific real data points, validating \cref{thm:path_geometry_straight}.}
\label{fig:exp_path_geometry_straight_trajectory}
\end{figure}

\paragraph{Hierarchical datasets.}

In this setting, we also set \( d = D = 2 \) for visualization purposes. Four cluster centers are defined at coordinates \((-2, 2)\), \((-2, -2)\), \((2, 2)\), and \((2, -2)\). Around each center, \(30\) points are sampled by adding Gaussian noise with a standard deviation of \(0.5\), resulting in a total of \( N = 120 \) real data points. Both \( \hat{\mO}_t \) and \( \hat{\vs}_t(\vx) \) are configured to their optimal versions to apply the optimal velocity field. In \cref{fig:exp_path_geometry_hierarchical_trajectory}, we plot the generation paths and intermediate points \(\vphi_t(\vx)\) for \( t = 0, 0.25, 0.5, \) and \( 0.75 \). The generation paths are derived by discretizing the flow ODE into \( 100 \) steps, starting from various noise samples (green crosses). These paths, represented by colored trajectories, evolve over time \( t \), with brighter colors indicating larger \( t \) values. Compared to those in \cref{fig:exp_path_geometry_straight_trajectory}, these paths are more curved. In the second, third, and fourth subfigures, the intermediate points \(\vphi_t(\vx)\) for \( t = 0.25, 0.5, \) and \( 0.75 \) are represented with markers in different colors and styles: cyan squares for \( t=0.25 \), orange diamonds for \( t=0.5 \), and red pentagons for \( t=0.75 \). The regions covered by these points overlap for \( t = 0, 0.25, \) and \( 0.5 \), but become disjoint at \( t = 0.75 \), effectively validating \cref{thm:path_geometry_hierarchical}.

\begin{figure}[htb]
\centering
\includegraphics[width=1\linewidth]{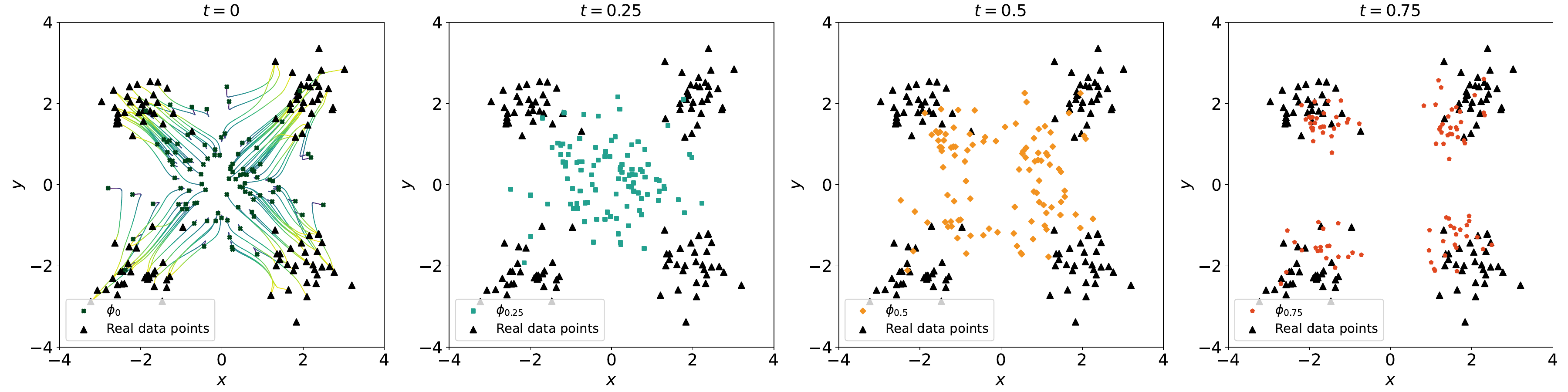}
\caption{The generation paths of Flow Matching models applied to hierarchical datasets (with one hierarchy) under the optimal velocity field. The black triangles represent real data points, which are obtained by adding Gaussian noise with a standard deviation of \(0.5\) to the four cluster centers \((-2, 2)\), \((-2, -2)\), \((2, 2)\), and \((2, -2)\). In the first subfigure (\( t=0 \)), the \textcolor[rgb]{0.0, 0.27, 0.11}{green crosses} show the initial noise samples, and the colored trajectories depict generation paths with brighter colors indicating larger \( t \). These paths are more curved than those in \cref{fig:exp_path_geometry_straight_trajectory}. In the second, third, and fourth subfigures, the markers represent the intermediate points \(\vphi_t(\vx)\): \textcolor[rgb]{0.14, 0.63, 0.56}{cyan squares} for \( t=0.25 \), \textcolor[rgb]{0.95, 0.58, 0.13}{orange diamonds} for \( t=0.5 \), and \textcolor[rgb]{0.88, 0.29, 0.13}{red pentagons} for \( t=0.75 \). The regions covered by these intermediate points overlap for \( t=0, 0.25, \) and \( 0.5 \), but are disjoint at \( t=0.75 \), providing visual validation for \cref{thm:path_geometry_hierarchical}.}
\label{fig:exp_path_geometry_hierarchical_trajectory}
\end{figure}

\subsection{Decay of Off-Subspace Components}

In this subsection, we first provide visualizations of the theoretical approximation results for \(\hat{\mO}_t\), followed by evidence of the decay of off-subspace components during actual training. These visualizations aim at supporting the theoretical analysis in \cref{subsec:decay_of_off_subspace_components}.

\paragraph{Sinusoidal positional encoding approximation.}

Recall that in our formulation, \(\hat{\mO}_t\) is a diagonal matrix where each diagonal entry is given by \(\hat{o}_t = \vkappa^\top \emb(t)\). Here, \(\vkappa\) is a parameter vector. The sinusoidal positional encoding, \(\emb(t)\), is defined as
\[
\emb(t) = \bigg( \sin\Big(\frac{s\cdot t}{\ell^{\frac{0}{\text{dim}}}}\Big), \cos\Big(\frac{s\cdot t}{\ell^{\frac{0}{\text{dim}}}}\Big), \sin\Big(\frac{s\cdot t}{\ell^{\frac{2}{\text{dim}}}}\Big), \cos\Big(\frac{s\cdot t}{\ell^{\frac{2}{\text{dim}}}}\Big), \dotsc, \sin\Big(\frac{s\cdot t}{\ell^{\frac{\text{dim}-2}{\text{dim}}}}\Big), \cos\Big(\frac{s\cdot t}{\ell^{\frac{\text{dim}-2}{\text{dim}}}}\Big) \bigg)^\top,
\]
where \(s\) is the scaling factor, \(\ell\) denotes the wavelength and \(\text{dim}\) represents the embedding dimension. In \cref{subsec:decay_of_off_subspace_components}, we have derived the limiting behavior of the off-subspace component of the generated sample as \(\vkappa\) converges to its limit value, \(-\mA^{-1}\vb\), where
\[
\mA = \int_0^1 (1 - t)^2 \emb(t) \emb(t)^\top \dif t, \quad \vb = \int_0^1 (1 - t) \emb(t) \dif t.
\]
To gain insight into the approximation behavior, we present a comparison of the limit function \(\vkappa^\top\emb(t)\) against the theoretical target \(-1/(1-t)\) under various configurations of the scaling factor \(s\) and embedding dimension \(\text{dim}\), as shown in \cref{fig:exp_sinusoidal_positional_encoding}. The plots display the absolute values of both functions on a logarithmic \(y\)-axis, highlighting differences across several orders of magnitude. The left plot corresponds to a scaling factor of \(s = 1\), while the right plot corresponds to \(s = 1000\), both with \(\ell\) fixed at \(10000\). Additionally, we include an inset zoomed around \(t = 0\) to provide a detailed view of the approximation behavior in the range \(t = 0\) to \(t = 0.1\). This magnified view offers further insight into the early-stage behavior of \(\vkappa^\top\emb(t)\) in relation to the theoretical target \(-1/(1-t)\). When \(s=1\), the approximation of \(\vkappa^\top\emb(t)\) deviates significantly from the target \(-1/(1-t)\) as \(t\) approaches \(1\), with all embedding dimensions showing similar behavior, indicating limited impact of embedding dimension on approximation quality. In contrast, when \(s=1000\), the approximation of \(\vkappa^\top\emb(t)\) remains close to the target over the majority of the range, reflecting a substantial improvement. However, for larger embedding dimensions, \(\vkappa^\top\emb(t)\) exhibits noticeable oscillations near \(t=0\), despite an overall better fit to the target.

\begin{figure}[htb]
\centering
\includegraphics[width=0.48\linewidth]{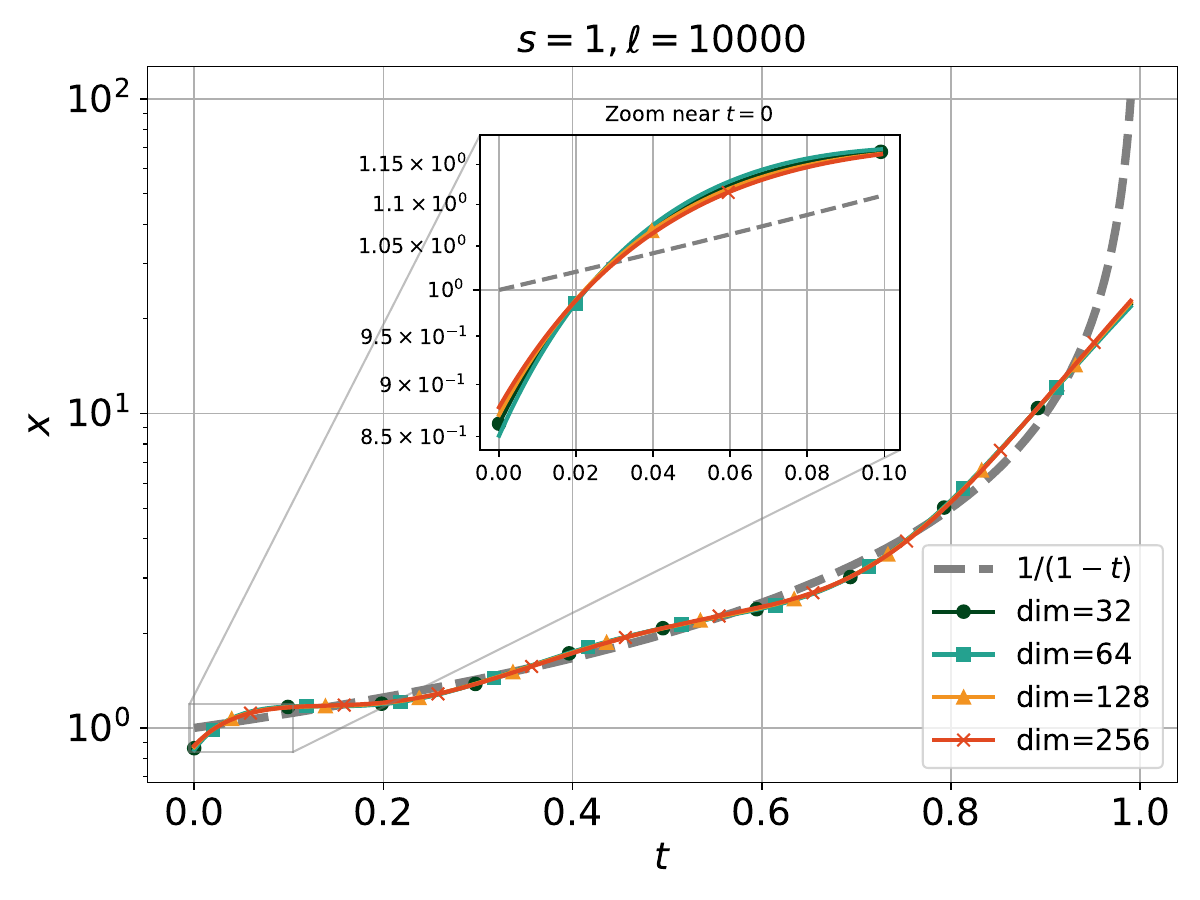}
\includegraphics[width=0.48\linewidth]{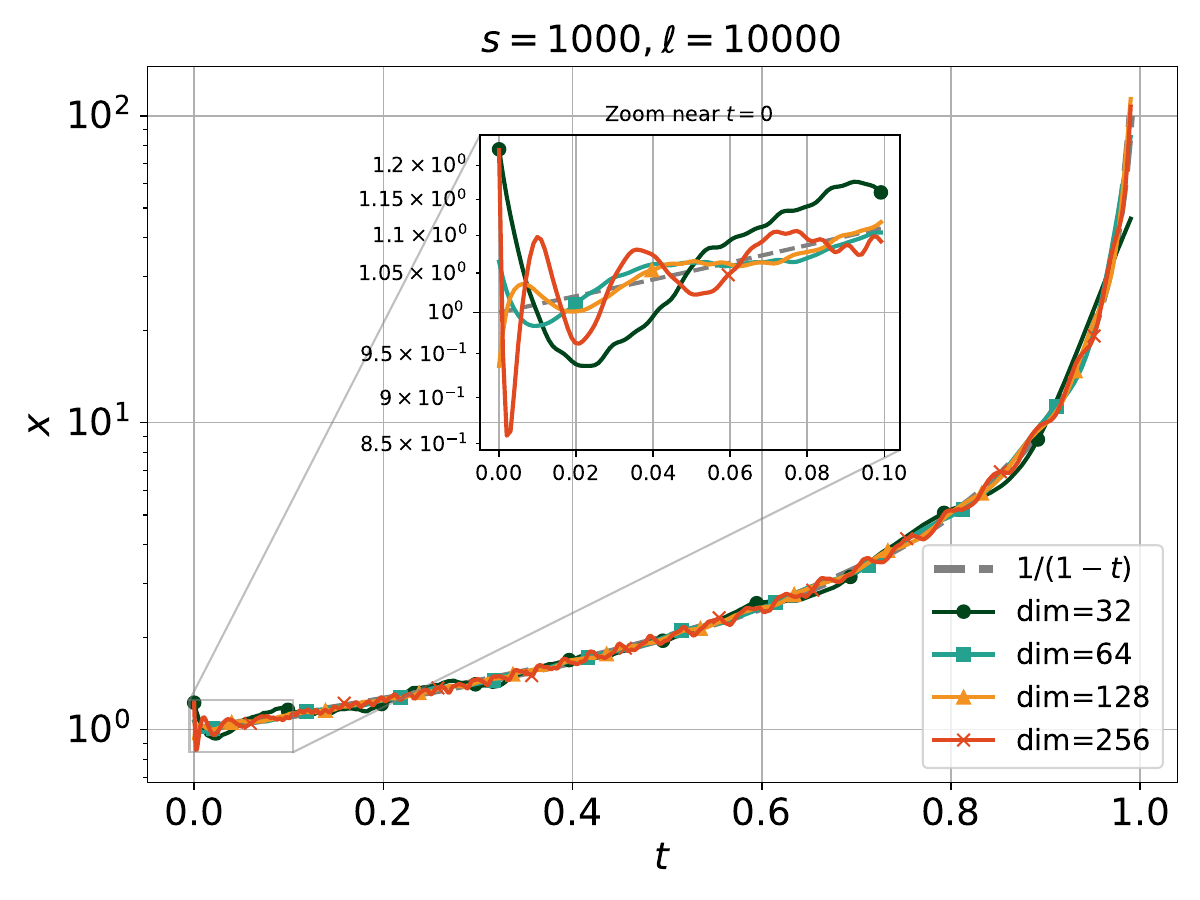}
\caption{Comparison of the limit function \(\vkappa^\top \emb(t)\) (solid colored lines) with the theoretical target \(-1/(1-t)\) (dashed line), both in absolute value, under different embedding dimensions. The \(y\)-axis is on a logarithmic scale. The left plot uses scaling factor \(s = 1\), while the right uses \(s = 1000\). Embedding dimensions are represented as follows: \(\text{dim} = 32\) (\textcolor[rgb]{0.0, 0.27, 0.11}{green, circles}), \(\text{dim} = 64\) (\textcolor[rgb]{0.14, 0.63, 0.56}{cyan, squares}), \(\text{dim} = 128\) (\textcolor[rgb]{0.95, 0.58, 0.13}{orange, triangles}), and \(\text{dim} = 256\) (\textcolor[rgb]{0.88, 0.29, 0.13}{red, crosses}). When \(s=1\), the approximation quality near \(t = 1\) is limited and largely independent of embedding dimension. For \(s=1000\), the approximation improves overall, but higher embedding dimensions introduce oscillations near \(t=0\).}
\label{fig:exp_sinusoidal_positional_encoding}
\end{figure}

\paragraph{Magnitude of off-subspace components.} In this experiment, we set \(d=100\) and \(D=20\). We randomly sample \(N=200\) points uniformly from the \(100\)-dimensional unit cube, and then set the last \(100-20=80\) entries to zero. The \(\hat{\vs}_t(\vx)\) is configured to its optimal version, while \(\hat{\mO}_t\) is implemented as a linear transformation of \(\emb(t)\) with \(s=1000\), \(\ell=10000\) and \(\text{dim}=256\). The OSDNet is trained using the SGD optimizer with a learning rate of \(0.1\) for \(80000\) epochs. Every \(20000\) epochs, we generate \(10000\) samples using a fixed random seed, and then plot the histogram of the Frobenius norm of the last \(80\) entries of these generated samples in \cref{fig:exp_frobenius_norm_histograms}. As shown in this figure, the Frobenius norm distribution tends to follow a normal distribution with a mean that gradually decreases over time. However, the norms do not vanish, indicating that the magnitude of the off-subspace components remains nonzero even after extensive training. These findings align with the theoretical discussions in \cref{subsec:decay_of_off_subspace_components}, where the decay of the off-subspace components is expected to reach a lower bound rather than diminish completely.

\begin{figure}[htb]
\centering
\includegraphics[width=1\linewidth]{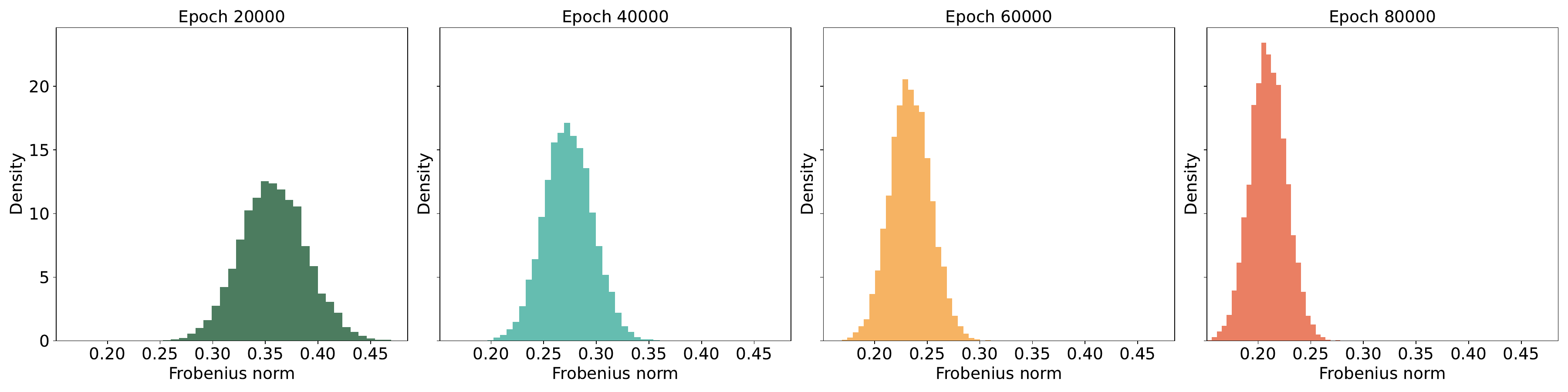}
\caption{Histograms of the Frobenius norm of the last \(80\) components of the \(10000\) generated samples at different training epochs (\textcolor[rgb]{0.0, 0.27, 0.11}{\(20000\)}, \textcolor[rgb]{0.14, 0.63, 0.56}{\(40000\)}, \textcolor[rgb]{0.95, 0.58, 0.13}{\(60000\)}, and \textcolor[rgb]{0.88, 0.29, 0.13}{\(80000\)} epochs). The Frobenius norm distribution tends to follow a normal distribution with a gradually decreasing mean. However, the norms do not converge to zero, indicating that the off-subspace components retain some magnitude even after extensive training. Both observations align with the discussions in \cref{subsec:decay_of_off_subspace_components}.}
\label{fig:exp_frobenius_norm_histograms}
\end{figure}

\subsection{Generalization of Subspace Components}

In this section, we first provide visualizations of the generated samples and their progression during training, followed by the relationship between MSE and training loss. These results illustrate the alignment between empirical observations and the theoretical analysis in \cref{thm:generalization_of_subspace_components}.

\paragraph{Visualization of generated samples.} In this experiment, we set \(d = D = 20\). We randomly sample \(N=200\) points uniformly from the \(20\)-dimensional unit cube. The \(\hat{\mO}_t\) is configured to its optimal version, while \(\hat{\vs}_t(\vx)\) is implemented using a ResNet with \(2\) residual blocks. Each block consists of fully connected layers with a hidden dimension of \(256\), followed by skip connections that allow the original input to be added to the output of the block. The OSDNet is trained using the AdamW optimizer with a learning rate of \(0.0001\) for \(20000\) epochs. Every \(5000\) epochs, we generate \(10000\) samples using a fixed random seed, and then plot the first two dimensions, as well as the \(2\)-dimensional Uniform Manifold Approximation and Projection (UMAP) \citep{mcinnes2018umap} embedding in \cref{fig:exp_generated_samples_and_umaps}. For clarity, only \(1000\) generated samples are displayed in the plots. As shown in this figure, the generated samples initially exhibit generalization that is distinct from the real data points. As training progresses, they gradually move closer to the real data distribution, aligning with the discussions in \cref{subsec:generalization_of_subspace_components}, where we demonstrate the generalization of subspace components under a suboptimal velocity field.

\begin{figure}[ht]
\centering
\includegraphics[width=1\linewidth]{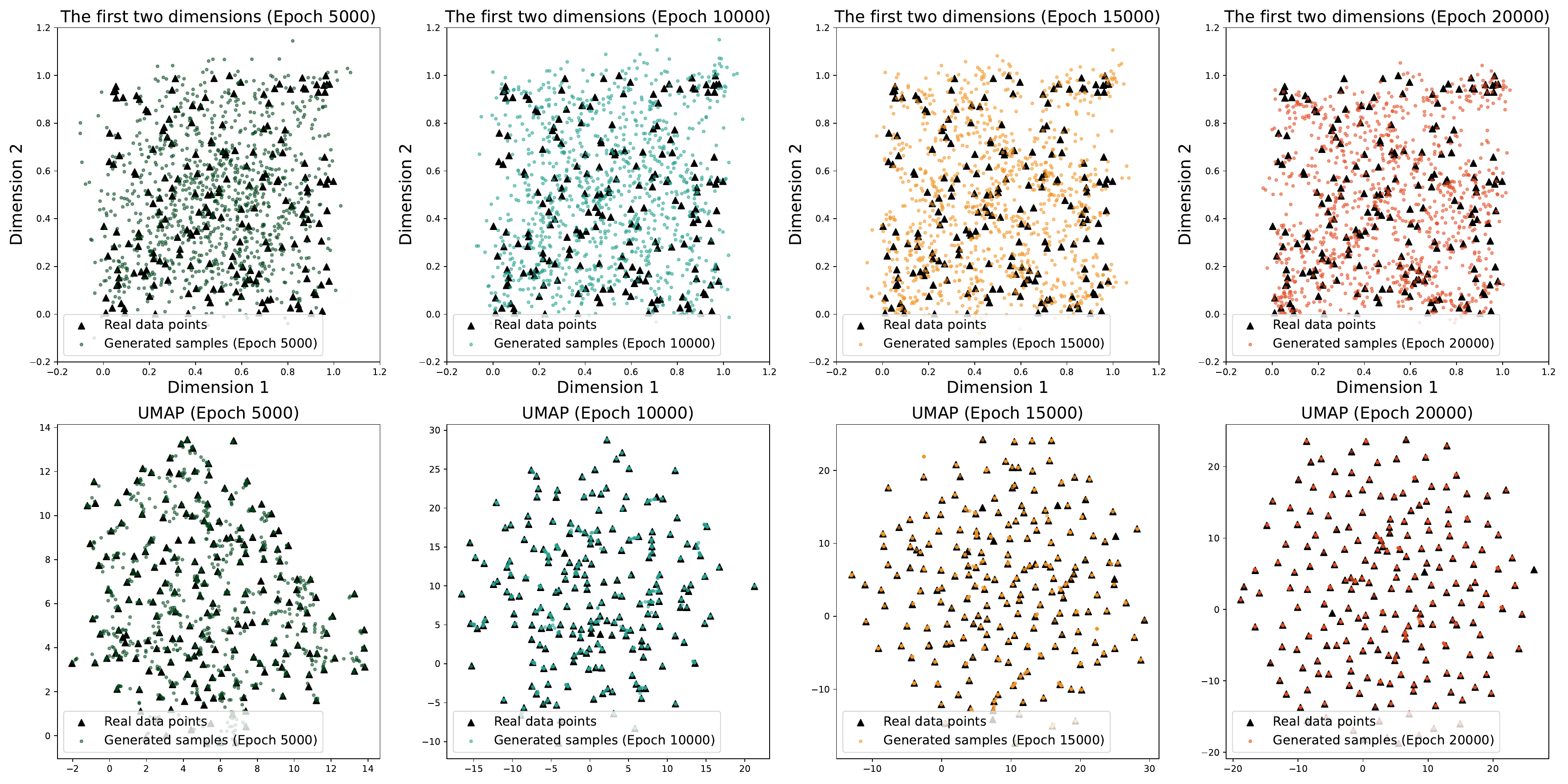}
\caption{The generated samples during training across different epochs. The first row visualizes the first two dimensions of the generated samples (colored dots) alongside the real data points (black triangles) at epochs \textcolor[rgb]{0.0, 0.27, 0.11}{\(5000\)}, \textcolor[rgb]{0.14, 0.63, 0.56}{\(10000\)}, \textcolor[rgb]{0.95, 0.58, 0.13}{\(15000\)}, and \textcolor[rgb]{0.88, 0.29, 0.13}{\(20000\)}, respectively. The second row presents the corresponding UMAP embeddings, reducing the dimensionality of the generated samples to a \(2\)-dimensional space. It can be observed that the generated samples exhibit generalization that is distinct from the real data points. As training progresses, they increasingly approach the real data points (which is more evident from the UMAP plots). Both observations align with the discussions in \cref{subsec:generalization_of_subspace_components}.}
\label{fig:exp_generated_samples_and_umaps}
\end{figure}

\paragraph{Relationship between MSE and Training Loss.} Every \(200\) epochs, we generate \(10000\) samples using both the OSDNet and its optimal counterpart. We then calculate the mean squared error (MSE) between the two sets of samples and plot these error values alongside the training loss in \cref{fig:exp_mse_vs_training_loss}. In this figure, the two curves represent the mean squared error (MSE) between the generated samples and the real data points, and the training loss, respectively, plotted over the training epochs. It can be observed that both curves exhibit similar trends, indicating that as the training progresses, the MSE decreases in parallel with the reduction in training loss. This supports the conclusion in \cref{thm:generalization_of_subspace_components}, where it is shown that the MSE is upper-bounded by the training error (up to a constant scaling), implying that the better the model is trained, the closer the generated samples are to the real data points on average.

\begin{figure}[htb]
\centering
\includegraphics[width=0.6\linewidth]{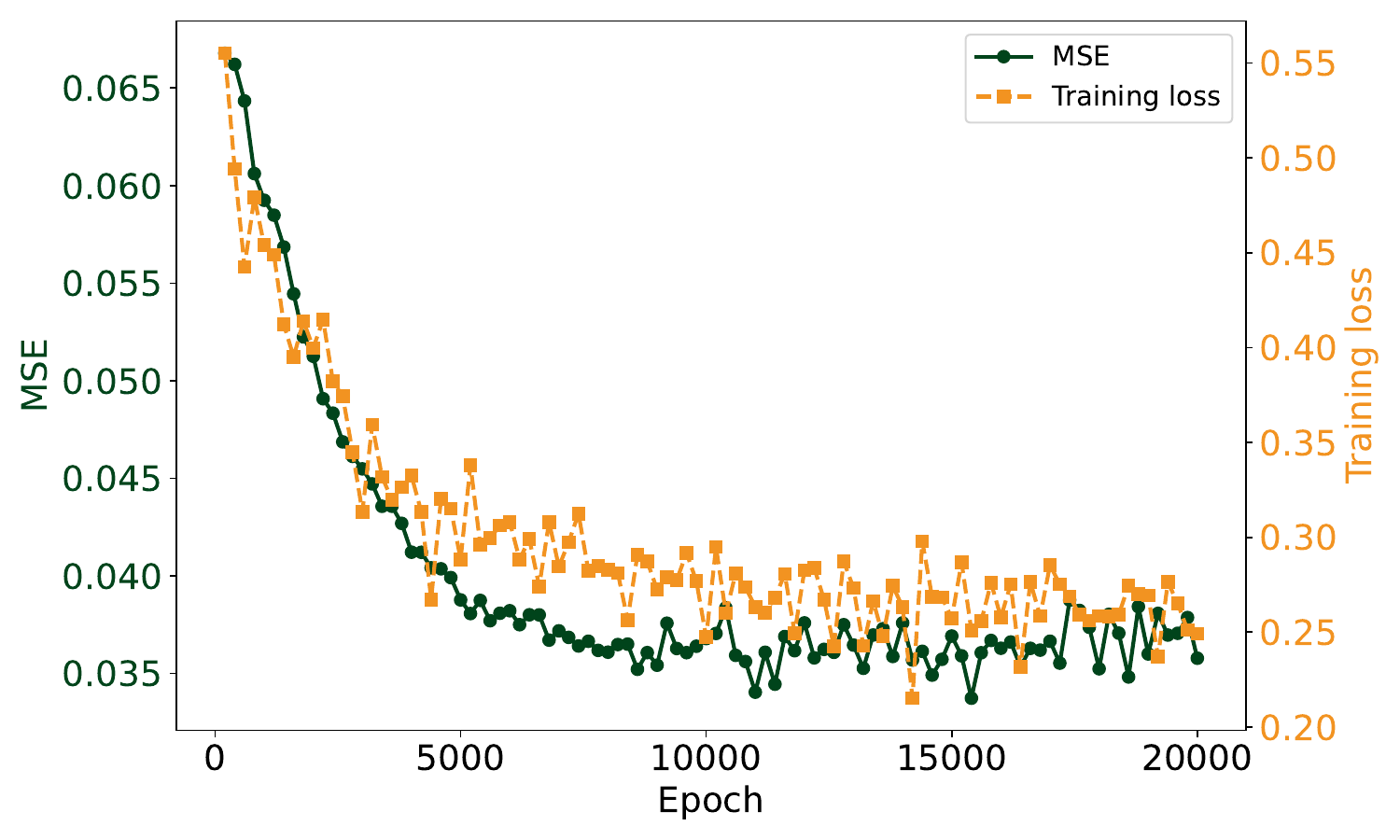}
\caption{The MSE between the generated samples and real data points (\textcolor[rgb]{0.0, 0.27, 0.11}{green, solid line with circle markers}), and the training loss (\textcolor[rgb]{0.95, 0.58, 0.13}{orange, dashed line with square markers}), plotted over the training epochs. Both curves show a similar trend, supporting the conclusion in \cref{thm:generalization_of_subspace_components} that the MSE is upper-bounded by the training error (up to a constant scaling), indicating that as training improves, the generated samples get closer to the real data points.}
\label{fig:exp_mse_vs_training_loss}
\end{figure}

\end{document}